\documentclass[11pt]{article} 

\usepackage[letterpaper,margin=1in]{geometry}

\newif\ifred
\redtrue 

\usepackage[T1]{fontenc}
\usepackage{microtype}

\usepackage{titlesec}
\titleformat*{\paragraph}{\bfseries}

\usepackage{amsthm,amsmath,amssymb,amsfonts,amssymb,mathtools}
\usepackage{amsmath,amssymb,amsfonts,amssymb,mathtools}
\usepackage{xcolor}
\usepackage{empheq}
\usepackage{dsfont}
\usepackage{mathrsfs}
\usepackage{graphicx}
\usepackage{enumitem}
\usepackage{aliascnt} 
\usepackage{xspace}
\usepackage{bbm}

\ifred
\newcommand{\new}[1]{#1}
\else
\newcommand{\new}[1]{#1}
\fi

\usepackage{caption}
\usepackage{tikz}
\usetikzlibrary{patterns}
\usetikzlibrary{patterns}
\usetikzlibrary{arrows,shapes,automata,backgrounds,petri,positioning}
\usetikzlibrary{shadows}
\usetikzlibrary{calc}
\usetikzlibrary{spy}
\usetikzlibrary{angles, quotes}
\usetikzlibrary{matrix}
\usepackage{pgf,pgfplots}
\usepackage{pgfmath,pgffor}
\pgfplotsset{compat=1.17}

\usepackage{framed}
 \usepackage[noend]{algorithm2e}
\usepackage[noend]{algpseudocode}

\definecolor[named]{ACMBlue}{cmyk}{1,0.1,0,0.1}
\definecolor[named]{ACMYellow}{cmyk}{0,0.16,1,0}
\definecolor[named]{ACMOrange}{cmyk}{0,0.42,1,0.01}
\definecolor[named]{ACMRed}{cmyk}{0,0.90,0.86,0}
\definecolor[named]{ACMLightBlue}{cmyk}{0.49,0.01,0,0}
\definecolor[named]{ACMGreen}{cmyk}{0.20,0,1,0.19}
\definecolor[named]{ACMPurple}{cmyk}{0.55,1,0,0.15}
\definecolor[named]{ACMDarkBlue}{cmyk}{1,0.58,0,0.21}

\usepackage[colorlinks,citecolor=blue,linkcolor=magenta,bookmarks=true]{hyperref}
\usepackage[nameinlink]{cleveref}
\crefname{ineq}{Inequality}{Inequality}
\creflabelformat{ineq}{#2{\upshape(#1)}#3}
\crefname{sub}{Subsection}{Subsection}
\creflabelformat{Subsection}{#2{\upshape(#1)}#3}
\crefname{sdp}{SDP}{SDP}
\creflabelformat{sdp}{#2{\upshape(#1)}#3}
\crefname{lp}{LP}{LP}
\creflabelformat{lp}{#2{\upshape(#1)}#3}
\usepackage{aliascnt}
\usepackage{cleveref}
\crefname{ineq}{Inequality}{Inequality}
\creflabelformat{ineq}{#2{\upshape(#1)}#3}
\crefname{sub}{Subsection}{Subsection}
\creflabelformat{Subsection}{#2{\upshape(#1)}#3}
\crefname{sdp}{SDP}{SDP}
\creflabelformat{sdp}{#2{\upshape(#1)}#3}
\crefname{lp}{LP}{LP}
\creflabelformat{lp}{#2{\upshape(#1)}#3}

\makeatletter
\newenvironment{Ualgorithm}[1][htpb]{\def\@algocf@post@ruled{\kern\interspacealgoruled\hrule  height\algoheightrule\kern3pt\relax}\def\@algocf@capt@ruled{under}\setlength\algotitleheightrule{0pt}\SetAlgoCaptionLayout{centerline}\begin{algorithm}[#1]}
{\end{algorithm}}
\makeatother

\DeclarePairedDelimiter\ceil{\lceil}{\rceil}

\newcommand{\CS}{Cauchy-Schwarz\xspace}

\newtheorem{theorem}{Theorem}[section]

\newtheorem{lemma}[theorem]{Lemma}
\newtheorem{informal theorem}[theorem]{Theorem (informal statement)}

\newtheorem{proposition}[theorem]{Proposition}

\newtheorem{claim}[theorem]{Claim}
\newtheorem{fact}[theorem]{Fact}

\newtheorem{remark}[theorem]{Remark}

\newtheorem{definition}[theorem]{Definition}

\newcommand{\lp}{\left}
\newcommand{\rp}{\right}

\newcommand\snorm[2]{\left\| #2 \right\|_{#1}}
\renewcommand\vec[1]{\mathbf{#1}}
\DeclareMathOperator*{\pr}{\mathbf{Pr}}
\DeclareMathOperator*{\E}{\mathbf{E}}

\def\d{\mathrm{d}}
\newcommand{\normal}{\mathcal{N}}

\DeclareMathOperator*{\argmin}{argmin}

\newcommand{\bx}{\mathbf{x}}

\newcommand{\bw}{\mathbf{w}}

\newcommand{\R}{\mathbb{R}}

\newcommand{\N}{\mathbb{N}}
\newcommand{\eps}{\epsilon}

\newcommand{\poly}{\mathrm{poly}}
\newcommand{\polylog}{\mathrm{polylog}}

\newcommand{\sgn}{\mathrm{sign}}
\newcommand{\sign}{\mathrm{sign}}

\newcommand{\OPT}{\mathrm{opt}}
\newcommand{\D}{D}

\newcommand{\wt}{\widetilde}

\newcommand{\x}{\vec x}
\newcommand{\w}{\vec w}

\newcommand{\opt}{\mathrm{opt}}
\newcommand{\Cagn}{C_{\mathrm{A}}}
\newcommand{\iid}{{i.i.d.}\ }
\newcommand{\abs}[1]{\lp| #1 \rp|}
\DeclareMathOperator{\Var}{Var}
\renewcommand\Pr{\pr}
\title{Efficient Testable Learning of Halfspaces \\ with Adversarial Label Noise}
\author{
Ilias Diakonikolas\thanks{Supported by NSF Medium Award CCF-2107079,
NSF Award CCF-1652862 (CAREER), a Sloan Research Fellowship, and
a DARPA Learning with Less Labels (LwLL) grant.}\\
UW Madison\\
{\tt ilias@cs.wisc.edu}\\
\and
Daniel M. Kane\thanks{Supported by NSF Award CCF-1553288 (CAREER) and a Sloan
  Research Fellowship.}\\
UC San Diego\\
{\tt dakane@ucsd.edu }\\
\and
Vasilis Kontonis\thanks{Supported in part by NSF Award CCF-2144298 (CAREER).}\\
UW Madison\\
{\tt kontonis@wisc.edu } \\
\and
Sihan Liu\\
UC San Diego\\
{\tt sil046@ucsd.edu}
\and
Nikos Zarifis\thanks{Supported in part by NSF Award CCF-1652862 (CAREER) 
and a DARPA Learning with Less Labels (LwLL) grant.}\\
UW Madison\\
{\tt zarifis@wisc.edu}\\
}

\begin{document}

\maketitle

\begin{abstract}
We give the first polynomial-time algorithm for the testable learning 
of halfspaces in the presence of adversarial label noise under the 
Gaussian distribution. In the recently introduced testable learning 
model, one is required to produce a 
tester-learner such that if the data passes the tester, 
then one can trust the output of the robust learner on the data.
Our tester-learner runs in time $\poly(d/\eps)$ and outputs a 
halfspace with misclassification error $O(\opt)+\eps$, where $\opt$ is 
the 0-1 error of the best fitting halfspace. 
At a technical level, our algorithm employs an iterative soft localization technique 
enhanced with appropriate testers to ensure that the data distribution is sufficiently 
similar to a Gaussian.

\end{abstract}

\setcounter{page}{0}
\thispagestyle{empty}
\newpage

\section{Introduction}

A (homogeneous) halfspace is a
Boolean function $h: \R^d \to \{ \pm 1\}$ of the form 
$h_{\bw}(\bx) = \sgn \left( \bw \cdot \bx \right)$, 
where $\bw \in \R^d$ is the corresponding weight vector and 
the function $\sign: \R \to \{ \pm 1\}$ is defined as $\sgn(t)=1$ if $t \geq 0$ 
and $\sgn(t)=-1$ otherwise.
Learning halfspaces from random labeled examples is a classical task in machine learning, 
with history going back to the Perceptron algorithm~\cite{Rosenblatt:58}. 
In the realizable PAC model~\cite{val84}
(i.e., with consistent labels), the class of halfspaces
is known to be efficiently learnable without distributional assumptions.
On the other hand, in the agnostic (or adversarial label noise) model~\cite{Haussler:92, KSS:94} even {\em weak} learning is computationally intractable in the distribution-free setting~\cite{Daniely16, DKMR22b, Tiegel22}.

These intractability results have served as a motivation for the study of agnostic learning in the distribution-specific setting, i.e., when the marginal distribution on examples is assumed to be well-behaved. In this context, a number of algorithmic results are known. 
The $L_1$-regression algorithm of~\cite{KKMS:08} agnostically learns halfspaces 
within near-optimal 0-1 error of $\opt+\eps$, where $\opt$ is the 0-1 error of the best-fitting halfspace.
The running time of the $L_1$-regression algorithm is $d^{\tilde{O}(1/\eps^2)}$ under the assumption
that the marginal distribution on examples is the standard Gaussian (and for a few other structured distributions)~\cite{DGJ+:10, DKN10}. While the $L_1$ regression method leads to improper learners, a proper agnostic learner with qualitatively similar complexity was recently given in~\cite{DKKTZ21}. 
The exponential dependence on $1/\eps$ in the running time of these algorithms is known to be inherent, in both the Statistical Query model~\cite{DKZ20, GGK20, DKPZ21} and under standard cryptographic assumptions~\cite{DKR23}.

Interestingly, it is possible to circumvent the super-polynomial dependence on $1/\eps$ by relaxing the final error 
guarantee --- namely, by obtaining a hypothesis with 0-1 error $f(\opt) +\eps$, for some function $f(t)$ that goes
to $0$ when $t \rightarrow 0$. (Vanilla agnostic learning corresponds to the case that $f(t)=t$.) 
A number of algorithmic works, 
starting with~\cite{KLS:09jmlr}, developed 
efficient algorithms with relaxed error guarantees; see, e.g.,~\cite{ABL17, Daniely15, DKS18a, DKTZ20c}. 
The most relevant results in the context of the current paper are the works~\cite{ABL17, DKS18a} which gave $\poly(d/\eps)$ time algorithms with error $C \opt+\eps$, for some universal constant $C>1$, for learning halfspaces with adversarial label noise under the Gaussian distribution. Given the aforementioned computational hardness results, these constant-factor approximations are best possible within the class of polynomial-time algorithms.

A drawback of distribution-specific agnostic learning is that it provides 
no guarantees if the assumption on the marginal distribution 
on examples is not satisfied. 
Ideally, one would additionally like an efficient method to {\em test} these 
distributional assumptions, so that: 
(1) if our tester accepts, then we can trust the output of the learner, 
and (2) it is unlikely that the tester rejects if the data satisfies the distributional assumptions. 
This state-of-affairs motivated the definition of a new model --- 
introduced in~\cite{RV22a} and termed {\em testable learning} --- formally
defined below:

\begin{definition}[Testable Learning with Adversarial Label Noise~\cite{RV22a}]
Fix $\epsilon, \tau \in (0,1]$ and let $f:[0,1] \mapsto \R_+$.
A tester-learner $\cal{A}$ (approximately) {\em testably learns} a concept class 
$\mathcal C$ with respect to the distribution $D_\x$ on $\R^d$
 with $N$ samples and failure probability $\tau$ 
if for any distribution $D$ on $\R^d\times\{\pm1\}$,
the tester-learner $\cal A$ draws a set $S$ of $N$ i.i.d.\ samples from $D$ 
and either rejects $S$ or accepts $S$ and produces a hypothesis $h:\R^d \mapsto \{\pm1\}$. Moreover, the following conditions must be met:
\begin{itemize}
    \item (Completeness)  If $D$ truly has marginal $D_\x$, $\cal{A}$ accepts with probability at least $1-\tau$.
    \item (Soundness) 
    \new{The probability that} $\cal A$ accepts and outputs a hypothesis $h$ \new{for which}  $\pr_{(\x,y) \sim D}[h(\x) \neq y] > f(\OPT) +\eps$ ,  where 
    $\opt: = \min_{g \in \cal{C}} \pr_{(\x,y) \sim D}[g(\x) \neq y]$ \new{is at most $\tau$}.
\end{itemize}
The probability in the above statements is over the randomness of the sample $S$ 
and the internal randomness of the tester-learner $\cal A$.
\end{definition}

The initial work~\cite{RV22a} and the followup paper~\cite{Gollakota2022} focused
on the setting where $f(t) = t$ (i.e., achieving optimal error of $\opt+\eps$). These works developed general moment-matching based algorithms that yield testable learners for a range of concept classes, including halfspaces. For the class of halfspaces
in particular, they gave a testable agnostic learner under the Gaussian distribution with
sample complexity and runtime $d^{\tilde{O}(1/\eps^2)}$ --- essentially matching the complexity of the problem in the standard agnostic PAC setting (without the testable requirement).
Since the testable learning setting is at least as hard as the standard PAC setting, the aforementioned hardness results imply that the exponential complexity dependence in $1/\eps$ cannot be improved.

In this work, we continue this line of investigation. 
We ask whether we can obtain {\em fully polynomial time} testable learning algorithms with relaxed error guarantees --- ideally matching the standard (non-testable) learning setting. Concretely, we study the following question:
\begin{center}
{\em Is there a $\poly(d/\eps)$ time tester-learner for halfspaces with error $f(\opt)+\eps$? \\ 
Specifically, is there a constant-factor approximation?}
\end{center}
As our main result, we provide an affirmative answer to this question 
in the strongest possible sense --- 
by providing an efficient constant-factor approximate tester-learner.

\paragraph{Main Result}
Our main result is the first polynomial-time tester-learner
for homogeneous halfspaces with respect to the Gaussian distribution in the presence of adversarial label noise.
Formally, we establish the following theorem:

\begin{theorem} [Testable Learning Halfspaces under Gaussian Marginals]\label{thm:testable-learning-agnostic}
Let $\eps, \tau \in (0,1)$ and $\mathcal C$ be the class of homogeneous halfspaces on $\R^d$.
There exists a $\poly(d, 1/\eps) \log(1/\tau)$-time
tester-learner for $\mathcal C$ 
with respect to $\normal(\vec 0, \vec I)$ up to 0-1 error $O(\OPT) + \eps$, where $\OPT$ is the 0-1 error of the best fitting function in $\mathcal{C}$ and $\tau$ is the failure probability.
\end{theorem}

Before we provide an overview of our technical approach, some 
remarks are in order. \Cref{thm:testable-learning-agnostic} 
gives the first algorithm for testable learning of halfspaces 
that runs in $\poly(d/\eps)$ time and achieves dimension-independent 
error (i.e., error of the form $f(\opt)+\eps$, where $f$ satisfies 
$\lim_{t \rightarrow 0}f(t)=0$.) Moreover, the constant-factor approximation 
achieved is best possible, matching the known guarantees without 
the testable requirement and complexity lower bounds.  
Prior to our work, the only known result in the testable setting, 
due to~\cite{RV22a, Gollakota2022}, achieves error $\opt+\eps$ 
with complexity $d^{\poly(1/\eps)}$. A novel (and seemingly necessary) 
feature of our approach is that 
the testing components of our algorithm depend on the labels (as opposed to the
label-oblivious testers of ~\cite{RV22a, Gollakota2022}). 
As will be explained in the proceeding 
discussion, to prove \Cref{thm:testable-learning-agnostic} 
we develop a testable version of the well-known localization technique 
that may be of broader interest.

\paragraph{Independent Work} In concurrent and independent work,~\cite{GKSV23} 
gave an efficient tester-learner for homogeneous halfspaces under the Gaussian 
distribution (and strongly log-concave distributions) achieving dimension-independent 
error guarantees. Specifically, their algorithm achieves 0-1 error 
$O(k^{1/2} \opt^{1-1/k})$ with sample complexity and running time of 
$\poly(d^{\widetilde{O}(k)}, (1/\eps)^{\widetilde{O}(k)})$. 
That is, they obtain error $O(\opt^{c})$, where $c<1$ is a universal constant, in 
$\poly_c(d/\eps)$ time; and error $\new{\wt{O}(\opt)}$ 
in quasi-polynomial $(d/\eps)^{\polylog(d)}$ time.

\subsection{Overview of Techniques}

Our tester-learner is based on the well-known localization technique that has been used in the context of learning halfspaces with noise; see, e.g., 
\cite{ABL17,DKS18a}.  At a high-level, the idea of localization hinges on updating a given hypothesis by using ``the most informative'' examples, specifically examples that have very small margin with respect to the current hypothesis.  Naturally, the correctness of this
geometric technique leverages structural properties of the underlying 
distribution over examples, namely concentration, anti-concentration, and 
anti-anti-concentration properties (see, e.g., \cite{DKTZ20}). While the 
Gaussian distribution satisfies these properties, they are unfortunately 
hard to test.  In this work, we show that localization can be effectively 
combined with appropriate efficient testing routines to provide an efficient tester-learner.

\paragraph{Localization and a (Weak) {\em Proper} Testable Learner} 
Assume that we are given a halfspace defined by the unit vector 
$\vec w$ with small 0-1 error, namely
$\pr_{\x\sim D_\x}[\sign(\vec v^\ast \cdot\x) \neq \sgn(\vec w \cdot \x)] \leq \delta$, for some small $\delta>0$, where $\vec v^{\ast}$ is the unit 
vector defining an optimal halfspace.
The localization approach improves the current hypothesis, 
defined by $\vec w$, by considering the conditional distribution $D'$ 
on the points that fall in a thin slice around 
$\vec w$, i.e., the set of points $\x$ satisfying $|\vec w \cdot \x | \leq O(\delta)$.   
The goal is to compute a new (unit) weight vector $\vec w'$ that is close to an optimal halfspace, defined by $\vec v^\ast$,
with respect to $D'$, i.e., 
$\pr_{\x'\sim D'_\x}[\sign(\vec v^\ast \cdot\x') \neq \sgn(\vec w' \cdot \x')] 
\leq \alpha$, for an appropriate $\alpha>0$. 
We can then show that the halfspace defined by $\vec w'$ will be {\em closer} to the target halfspace (defined by $\vec v^\ast$) with respect to the {\em original} distribution, i.e., we have that 
$\pr_{\x\sim D_\x}[\sign(\vec v^\ast \cdot\x) \neq \sgn(\vec w' \cdot \x)] \leq O(\delta \alpha)$.  By repeating the above step, we  iteratively reduce 
the disagreement with $\vec v^\ast$ until we reach our target error of $O(\opt)$.
Similarly to \cite{DKS18a}, instead of ``hard'' conditioning on a thin slice, we  perform a ``soft'' localization step where (by rejection sampling) we transform the $\x$-marginal to a Gaussian
whose covariance is $O(\delta^2)$ in the direction of $\vec w$ and identity in 
the orthogonal directions, i.e., $\vec \Sigma = \vec I - (1-\delta^2) \vec w \vec w^\top$; see \Cref{lem:rejection-sampling}.  

A crucial ingredient of our approach is a \textbf{{\em proper} testable, weak agnostic learner}
with respect to the Gaussian distribution.  More precisely, our tester-learner runs in polynomial time and either
reports that the $\x$-marginal is not $\normal(\vec 0, \vec I)$ or outputs 
a unit vector $\vec w$ \new{with small constant} distance to the target $\vec v^\ast$, i.e.,
$\|\vec w - \vec v^\ast\|_2 \leq 1/100$; see \Cref{lem:constant-learn}.
Our weak proper tester-learner first verifies that the given $\x$-marginal
approximately matches 
constantly many low-degree moments with the standard Gaussian; and if it does, it 
returns the vector defined by the degree-$1$ Chow parameters, i.e.,
$\vec c = \E_{(\x, y) \sim \D}[y \vec x ]$.
Our main structural result in this context shows that if $D_\x$ approximately 
matches its low-degree moments with the standard Gaussian,  
then the Chow parameters of any homogeneous LTF 
with respect to $D_\x$ are close to its Chow parameters with respect to $\normal(\vec 0, \vec I)$, i.e.,
for any homogeneous LTF $f(\x)$, we have that 
$\E_{\x \sim D_{\bx}}[ f(\x)\x] \approx \E_{\x^\ast \sim \cal{N}(\vec{0}, \vec{I})}[f(\x^\ast)\x^\ast ]$; see \Cref{lem:moment-robust-chow}.
Since the Chow vector of a homogeneous LTF with respect to the Gaussian distribution 
 is parallel to its normal vector $\vec v^\ast$  (see \Cref{lem:unbiased-chow}),
 it is not hard to show that the Chow vector of the LTF \new{with respect to $D_{\vec x}$} will not be very far from 
 $\vec v^\ast$ and will satisfy the (weak) learning guarantee of 
 $\|\vec c - \vec v^\ast \|_2 \leq 1/100$.
Finally, \new{to deal with label noise}, we show that if $\x'$ has bounded second moments (a condition that we can efficiently test), we can robustly estimate $\E_{\x\sim D_\x}[ f(\x)\x]$ with samples from $D$ up to error $O(\sqrt{\opt})$ (see \Cref{lem:robust-chow}), which suffices for our purpose of weak learning.  The detailed description 
of our weak, proper tester-learner can be found in \Cref{sec:weak-proper-tester-learner}.

\paragraph{From Parameter Distance to Zero-One Error}
Having a (weak) testable proper learner, 
we can now use it on the localized 
(conditional) distribution $D'$ and obtain a vector $\vec w'$ 
that is closer to $\vec v^\ast$ in $\ell_2$ distance; see \Cref{lem:localization}.
However, our goal is to obtain a vector that has small zero-one disagreement 
with the target halfspace $\vec v^\ast$.  Assuming that the underlying $\x$-marginal
is a standard normal distribution, and that 
$\|\vec w - \vec v^\ast\|_2 = \delta$, it holds that 
$\pr_{\x\sim D_\x}[\sign(\vec w\cdot \x) \neq \sign(\vec v^\ast\cdot\x)]
= O(\delta)$, which implies that 
achieving $\ell_2$-distance $O(\opt) + \eps$ suffices.
We give an algorithm that can efficiently \new{either certify that}
small $\ell_2$-distance
implies small zero-one disagreement with respect to the given marginal $D_\x$
or declare that $D_\x$ is not the standard normal.

      \definecolor{mycyan}{RGB}{42,161,152}
\definecolor{myred}{RGB}{220,50,47}
\begin{figure}
	\centering
	\begin{tikzpicture}[scale=1, every node/.style={transform shape} ]
	\coordinate (start) at (0.5,0.5);
	\coordinate (center) at (0,0);
	\coordinate (end) at (0,0.7);
	\coordinate (end2) at (-0.12/0.7634,0.366/0.7634);
	\coordinate (start2) at (0,0.7);
	\coordinate (start3) at (3,0);
	\coordinate (end3) at (45:3);
	\coordinate (start4) at (135:2);
	\coordinate (end4) at (-2,0);
\draw[fill=mycyan, opacity=0.6] (0,0) -- (3,0) arc (0:45:3.0cm)--cycle;
 \draw[fill=mycyan, opacity=0.6] (0,0) -- (-2,0) arc (180:225:2.0cm)--cycle;

\draw[black,thick,-] (-2,0) -- (3,0) ;
 	\draw[black,dashed,-] (-2,0.5) -- (3,0.5) ;
  	\draw[black,dashed,-] (-2,1) -- (3,1) ;
    	\draw[black,dashed,-] (-2,1.5) -- (3,1.5) ;

\draw[blue,thick,-] (0,0.5) -- (3,0.5) ;
     \draw[blue,thick,-] (0.5,0.5) -- (0.5,1) ;
      \draw[blue,thick,-] (0.5,1) -- (3,1) ;
\draw[blue,thick,-] (1,1) -- (1,1.5) ;
           \draw[blue,thick,-] (1,1.5) -- (3,1.5) ;
     \draw[black,thick,-] (0,-1) -- (0,2.5) ;
     \draw[blue,dashed,-] (1,1.5)--(1,0)  node[below ]{$\Theta(i)$};
\draw[black,thick,->] (0,-0) -- (-0.45,0.45) node[anchor= south ] {$\vec v^\ast$};
 \draw[black,thick,->] (0,-0) -- (0.5,0) node[anchor= south ] {$\vec u$};
\draw[purple,thick,<->] (-2,0) -- (-2,0.5)node[anchor= south east,below left] {$\delta$};
\draw[purple,thick,<->] (-2,0.5) -- (-2,1)node[anchor= south east,below left] {$\delta$};
\draw[purple,thick,<->] (-2,1) -- (-2,1.5)node[anchor= south east,below left] {$\delta$};

	\draw[black] (-1,-1) -- (2,2);
	\draw[black,thick ,->] (0,0) -- (0,0.5) node[above right] {$\bw$};
\pic [draw, <->, angle radius=15mm, angle eccentricity=1.3, "$\theta = \Theta(\delta)$"] {angle = start3--center--end3};
\end{tikzpicture}
	\caption{The disagreement region between a halfspace with normal vector $\vec w$ and the target $\vec v^\ast$ is shown in green.
 The unit direction $\vec u$ corresponds to the projection of $\vec v^\ast$
 on the orthogonal complement of $\vec w$.
 We assume that the $\ell_2$ distance of the two halfspaces is $\delta$ (and thus their angle is $\Theta(\delta)$). Since the slabs $S_i = \{i \delta \leq |\vec x \cdot \vec w| \leq (i+1) \delta \}$ have width $\delta$, the $x$-coordinate of the start of the $i$-th box is $\Theta(i)$.
 }
	\label{fig:wedge}
\end{figure}
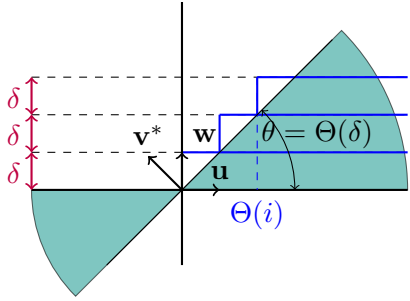 

The disagreement region of $\vec v^\ast$ and $\vec w$ is a union of two ``wedges'' (intersection of two halfspaces); 
see \Cref{fig:wedge}. \new{In order for our algorithm to work, we need to verify that these wedges do not contain too much probability mass.}
Similarly to our approach for the weak tester-leaner, one could try a moment-matching
approach and argue that if $D_\x$ matches its ``low''-degree moments with $\normal(\vec 0, \vec I)$,
then small $\ell_2$-distance translates to small zero-one disagreement.  
However, we will need to use this result for vectors that are very close to the target (but still not close enough), namely 
$\|\vec w - \vec v^\ast\|_2 = \delta$, where $\delta = \Theta(\eps)$; this would
require matching $\poly(1/\delta)$ many moments (as we essentially
need to approximate the wedge of \Cref{fig:wedge} with a polynomial)
and would thus lead to an exponential runtime of $d^{\poly(1/\delta)}$.

Instead of trying to approximate the disagreement region with a polynomial, we will make use of the fact that our algorithm knows $\vec w$ (but not $\vec v^\ast$) and approximate the disagreement
region by a union of cylindrical slabs.  We consider slabs of the form $S_i = \{\x : i \delta \leq 
|\vec w \cdot \x| \leq (i+1) \delta \}$. \new{If} the target distribution is Gaussian, we know that the set 
$|\vec w \cdot \x| \gg \sqrt{\log(1/\eps)}$ has mass $O(\delta)$
and we can essentially ignore it.  Therefore, we can cover the whole space by considering roughly $M =O(\sqrt{\log(1/\delta)}/\delta)$ slabs of width $\delta$ and split the disagreement region into
the disagreement region inside each slab $S_i$.  We have that 
\begin{align*}
\pr_{\x\sim D_\x}[\sign(\vec w\cdot \x) \neq \sign(\vec v^\ast\cdot\x)]
\leq 
\sum_{i=1}^M  \pr[|\vec u \cdot \x| \geq i \mid \x \in S_i] ~ \pr[S_i]
\,,
\end{align*}
where $\vec u$ is the unit direction parallel to the projection of 
the target $\vec v^\ast$ onto the orthogonal complement of $\vec w$,
see \Cref{fig:wedge}.
By the anti-concentration of 
the Gaussian distribution we know that each slab should have mass
at most $O(\delta)$.  Note that this is easy to test by sampling
and computing empirical estimates of the mass of each slab.
Moreover, assuming that underlying distribution is 
$\normal(\vec 0, \vec I)$, we have that, conditional on $S_i$ the orthogonal direction $\vec u\cdot \x \sim \normal(0, 1)$ (see \Cref{fig:wedge}) and in particular $\vec u \cdot \x$ has bounded 
second moment.  We do not know the orthogonal direction $\vec u$
as it depends on the unknown $\vec v^\ast$ but we can check that,
conditional on the slab $S_i$, the projection of $D_\x$ onto the 
orthogonal complement of $\vec w$ is (approximately) mean-zero and has 
bounded  covariance (i.e., bounded above by $2 \vec I$).
Note that both these conditions hold when $\x \sim \normal(\vec 0, \vec I)$ 
and can be efficiently tested with samples \new{in time $\poly(d, 1/\delta)$}. 
Under those conditions we have that that $\pr[S_i] = O(\delta)$
for all $i$.  Moreover, when the conditional distribution
on $S_i$ (projected on the orthogonal complement of $\vec w$)
has bounded second moment, we have that 
$
\pr[|\vec u \cdot \x| \geq i \mid \x \in S_i] 
\leq O(1/i^2)\,.
$
Combining the above, we obtain that under those assumptions 
the total probability of disagreement is at most $O(\delta)$.
The detailed analysis is given in \Cref{ssec:distances}.

\subsection{Preliminaries} \label{ssec:prelims}
We use small boldface characters for vectors and capital bold characters for matrices.
We use $[d]$ to denote the set $\{1, 2, \ldots, d\}$.
For a vector $\vec x \in \R^d$ and $i \in [d]$, $\vec x_i$ denotes the $i$-th coordinate of $\vec x$, and $\|\vec x\|_2 := \sqrt{ \sum_{i=1}^d \vec x_i^2 }$ the $\ell_2$ norm of $\vec x$.
We use $\vec x \cdot \vec y := \sum_{i=1}^n \vec x_i  \vec y_i$ as the inner product between them. We use $\mathbbm 1\{ E  \}$  to denote the indicator function of some event $E$. 

We use $\E_{\vec x \sim D}[\vec x]$ for the expectation of the random variable $\vec x$ according to the distribution $D$ and $\Pr[E]$ for the probability of event $E$. For simplicity of notation, we may omit the distribution
when it is clear from the context. 
For $\vec \mu \in \R^d, \vec \Sigma \in \R^{d \times d}$, we denote by $\normal (\vec \mu, \vec \Sigma)$  the $d$-dimensional Gaussian distribution
with mean $\vec \mu$ and covariance $\vec \Sigma$.
For $(\vec x, y) \in \mathcal X$ distributed according to $D$, we denote $D_{\vec x}$
to be the marginal distribution of $\vec x$.
Let $f: \R^d \mapsto \{\pm 1\}$ be a boolean function and $D$ a distribution over $\R^d$. The degree-$1$ Chow parameter vector of $f$ with respect to $D$ is defined as $\E_{\vec x \sim D}\lp[f(\vec x) \vec x  \rp]$. 
For a halfspace $h(\vec x) = \sgn(\vec v \cdot \vec x)$, we say that $\vec v$ is the defining vector of $h$.

\paragraph{Moment-Matching}  In what follows, we use the phrase 
\emph{``A distribution $D$ on $\R^d$ matches $k$ moments with a distribution $Q$ up to error $\Delta$''}.  Similarly to \cite{Gollakota2022}, we formally define approximate moment-matching as follows.
\begin{definition}[Approximate Moment-Matching]
Let $k\in \mathbb N$ be a degree parameter and let $\mathcal{M}(k, d)$ be the set of $d$-variate monomials of degree up to $k$.  Moreover, let 
$\vec \Delta \in \R_+^{|\mathcal{M}(k,d)|}$ be a slack parameter (indexed by the monomials of $\mathcal{M}(k,d)$), satisfying
$\vec \Delta_0 = 0$.  We say that two distributions $D, Q$ match $k$ moments up
to error $\vec \Delta$ if 
$|\E_{\x \sim D}[m(\x)] - \E_{\x \sim Q}[m(\x)]| \leq \vec \Delta_m$ for every monomial
$m(\x) \in \mathcal{M}(k,d)$.
When the error bound $\Delta$ is the same for all monomials we overload notation and simply use $\Delta$ instead of the parameter $\vec \Delta$.
\end{definition}

\section{Weak Testable Proper Agnostic Learning}
\label{sec:weak-proper-tester-learner}
As our starting point, 
we give an algorithm that performs testable \emph{proper} learning 
of homogeneous halfspaces in the presence of adversarial label noise 
with respect to the Gaussian distribution. 
The main result of this section is the following:

\begin{proposition}[Proper Testable Learner with Adversarial Label Noise]
\label{lem:constant-learn}
Let $D$ be a distribution on labeled examples $(\vec x, y) \in \R^d \times \{ \pm 1\}$.
Suppose that there exists a unit vector $\vec v^\ast\in \R^d$ such that 
$\Pr_{(\vec x, y) \sim D}\lp[ \sgn( \vec v^\ast \cdot \vec x )\neq y \rp] \leq \opt$.
There exists an algorithm (\Cref{alg:proper-testable-learner})
that given $\tau, \eta \in (0,1)$, and 
$N=d^{ \widetilde O(1 / \eta^2) } \log(1/\tau)$ \iid samples from $D$, 
runs in time $\poly(d,N)$ and does one of the following: 
\begin{itemize}[leftmargin=*]
    \item The algorithm reports that the $\x$-marginal of $D$ is not $\normal(\vec 0, \vec I)$.
\item The algorithm outputs a unit vector $\vec w \in \R^d$.
\end{itemize}
With probability at least $1 - \tau$ the following holds: 
(1) if the algorithm reports anything, the report is correct, and 
(2) if the algorithm returns a vector $\vec w$, it holds 
$\snorm{2}{\vec v^\ast - \vec w} \leq \Cagn  \sqrt{\OPT + \eta}$, 
where $\Cagn>0$ is an absolute constant.
\end{proposition}

A couple of remarks are in order. 
First, notice that if the algorithm outputs a vector $\vec w$, 
we only have the guarantee that $\snorm{2}{\vec v^\ast - \vec w}$ is small ---
instead of that the hypothesis halfspace $h_{\vec w}(\vec x) = \sgn(\vec w \cdot \vec x)$ achieves small 0-1 error. Nonetheless, as we will show in the next section, 
conditioned on $D$ \new{passing} some test, the error of the halfspace $h_{\vec w}$ 
will be at most $\opt$ plus a constant multiple of $\snorm{2}{\vec v^\ast - \vec w}$ (see \Cref{lem:wedge-bound}).
Second, unlike the testable {\em improper} learners in \cite{RV22a, Gollakota2022} 
--- which achieve error of $\OPT + \eta$ with similar running time and sample complexity ---  
our testable {\em proper} learner achieves the weaker error guarantee of 
$O( \sqrt{\OPT + \eta})$. This suffices for our purposes for the following reason:
in the context of our localization-based approach, 
we only need an efficient proper \emph{weak} learner that 
achieves sufficiently small {\em constant} error. This holds 
for our proper testable learner, 
as long as both $\opt$ and $\eta$ are bounded above 
by some other sufficiently small constant.

To obtain a proper learner, we proceed to directly estimate the defining vector $\vec v^\ast$ of the target halfspace $h^\ast (\x) = \sgn(\vec{v}^{\ast} \cdot \x)$, where we 
assume without loss of generality that $\vec v^\ast$ is a unit vector.
The following simple fact relating the degree-$1$ \emph{Chow-parameters} 
of a homogeneous halfspace and its defining vector will be useful for us.
\begin{fact}[see, e.g.,~Lemma 4.3 of~\cite{DKS18a}] \label{lem:unbiased-chow}
Let $\vec v$ be a unit vector and $h(\vec x) = \sgn(\vec v \cdot \vec x)$ 
be the corresponding halfspace.
If $\vec x$ is drawn from $\normal(\vec 0, \vec I)$, 
then  we have that 
$\E_{\vec x \sim \normal(\vec 0, \vec I)} \lp[ h(\vec x)  \vec x\rp] = \sqrt{2/\pi} ~ \vec v$. 
\end{fact}

To apply \Cref{lem:unbiased-chow} in our context, we need to overcome two hurdles: 
(i) the $\vec x$ marginal of $D$ is not necessarily the standard Gaussian, and 
(ii) the labels are not always consistent with $h^\ast(\vec x)$.
The second issue can be circumvented by following the approach of \cite{DKS18a}.
In particular, if the $\vec x$ marginal of $D$ is indeed Gaussian, 
we can just treat $D$ as a corrupted version of $(\vec x, h^\ast(\vec x))$, 
where $\vec x \sim \normal(\vec 0, \vec I)$ 
and estimate the Chow parameters robustly.

To deal with the first issue, we borrow tools from \cite{Gollakota2022}. 
At a high level, we certify that the low-degree moments of $D_{\x}$
--- the $\vec x$ marginal of $D$ --- approximately match 
the corresponding moments of $\normal(\vec 0, \vec I)$ 
before estimating the Chow parameters. 
To establish the correctness of our algorithm, 
we show that, for any distribution $B$ that passes the moment test, 
the Chow parameters of a halfspace under $B$ 
will still be close to its defining vector.
Formally, we prove the following lemma: 

\begin{lemma}[From Moment-Matching to Chow Distance]\label{lem:moment-robust-chow}
Fix $\eta > 0$.
Let $k=C\log(1/\eta)/\eta^2$ and $\Delta =  \frac{1}{k  d^k }  \lp(  \frac{1}{ C  \sqrt{k}  } \rp)^{k+1}$, where $C>0$ is a sufficiently large absolute constant.
Let $B$ be a distribution whose moments up to degree $k$ match with those 
of $\normal(\vec 0, \vec I)$ up to additive error $\Delta$. 
Let $h(\vec x) = \sgn(\vec v \cdot \vec x)$ be a halfspace.
Then we have that 
$$
\snorm{2}{\E_{\vec x \sim B} \lp[ h(\vec x)  \vec x \rp]
- \new{\sqrt{\frac{2}{\pi}}} ~ \vec v} \leq O(\sqrt{\eta}) \;.
$$
\end{lemma}
\begin{proof}
It suffices to show that for any unit vector $\vec u \in \R^d$, 
the following holds: 
\begin{align*}
\abs{\E_{\vec x \sim B} \lp[ h(\vec x)  \vec x \cdot \vec u \rp] 
- \E_{\vec x \sim\normal(\vec 0, \vec I)} \lp[ h(\vec x)  \vec x \cdot \vec u \rp] }  \leq O(\sqrt{\eta}) \;.
\end{align*}
The following fact expresses 
a real number $a$ as an integral of the $\sgn$ function.
\begin{fact} \label{fact:sign-int}
For any $a \in \R$, it holds
\new{
$ a = \frac{1}{2} \int_{0}^\infty (\sgn(a - t) + \sgn(a + t)) \d t  $.
}
\end{fact}
We apply \Cref{fact:sign-int} to the term $\vec u\cdot \x$, which gives
\begin{align*}
\abs{\E_{\vec x \sim B} \lp[ h(\vec x)  \vec x \cdot \vec u \rp]
- \E_{\vec x \sim\normal(\vec 0, \vec I)} \lp[ h(\vec x)  \vec x \cdot \vec u \rp]} 
&= 
\frac{1}{2}
\bigg|  \E_{\vec x \sim B} \lp[ h(\vec x)  \int_{t \geq 0} \lp( \sgn( \vec u \cdot \vec x - t) + \sgn( \vec u \cdot \vec x + t) \rp) \d t\rp] \\
&- \E_{\vec x \sim \normal(\vec 0, \vec I)} \lp[ h(\vec x)  \int_{t \geq 0} \lp( \sgn( \vec u \cdot \vec x - t) + \sgn( \vec u \cdot \vec x + t) \rp) \d t\rp] \bigg| \\
&= 
\frac{1}{2}
\bigg | \int_{t \geq 0}\bigg( \E_{\vec x \sim B} \lp[ h(\vec x)  \lp( \sgn( \vec u \cdot \vec x - t) + \sgn( \vec u \cdot \vec x + t) \rp) \rp] \\
&- \E_{\vec x \sim \normal(\vec 0, \vec I)} \lp[ h(\vec x)  \lp( \sgn( \vec u \cdot \vec x - t) + \sgn( \vec u \cdot \vec x + t) \rp) \rp] \bigg)\d t
\bigg | \,,
\end{align*}
where in the last line we switch the order of the integral of $t$ and $\vec x$ by Fubini's theorem.
We then split the above integral over $t$ into two parts based on the magnitude of $t$ 
($t > 1/\sqrt{\eta}$ versus $0\leq t \leq 1/\sqrt{\eta}$) 
and apply the triangle inequality:
\begin{align} 
&\abs{\E_{\vec x \sim B} \lp[ h(\vec x)  \vec x \cdot \vec u \rp]
- \E_{\vec x \sim\normal(\vec 0, \vec I)} \lp[ h(\vec x)  \vec x \cdot \vec u \rp]} \nonumber \\
& \leq 
\frac{1}{2}
\abs{ \int_{ 0\leq t \leq 1/\sqrt{\eta} } \bigg(\E_{\vec x \sim B} \lp[ h(\vec x)  \sgn( \vec u \cdot \vec x - t) \rp]
- \E_{\vec x \sim \normal(\vec 0, \vec I)} \lp[ h(\vec x)  \sgn( \vec u \cdot \vec x - t) \rp] \bigg)\d t} \nonumber \\
&+\frac{1}{2}
\abs{ \int_{ 0\leq t \leq 1/\sqrt{\eta} }\bigg( \E_{\vec x \sim B} \lp[ h(\vec x)  \sgn( \vec u \cdot \vec x + t) \rp]
- \E_{\vec x \sim \normal(\vec 0, \vec I)} \lp[ h(\vec x)  \sgn( \vec u \cdot \vec x + t) \rp]\bigg) \d t} \nonumber \\
&+
\frac{1}{2}
\bigg| \int_{ t \geq 1/\sqrt{\eta} }\bigg( \E_{\vec x \sim B} \lp[ h(\vec x)  
\lp(\sgn( \vec u \cdot \vec x - t) + \sgn( \vec u \cdot \vec x + t)\rp)
\rp] \nonumber \\
&- \E_{\vec x \sim \normal(\vec 0, \vec I)} \lp[ h(\vec x)  \lp(\sgn( \vec u \cdot \vec x - t) + \sgn( \vec u \cdot \vec x + t)\rp) \rp] \bigg)
\d t
\bigg|
\label{eq:split} \;.
\end{align}
We start by bounding the integral for $t \geq 1/\sqrt{\eta}$. 
\new{
\begin{lemma}[Chow-Distance Tail]\label{lem:chow-distance-tail}
Let $Q$ be distribution over $\R^d$ with 
$\E_{\x \sim Q}[\x \x^\top] \preccurlyeq 2 \vec I$.
Moreover, let $g(\x):\R^d \mapsto \R$ be a bounded function, i.e.,
$|g(\x)| \leq 1$ for all $\x \in \R^d$.
It holds
\begin{align*} 
&\abs{\int_{ t \geq 1/\sqrt{\eta} } \E_{\vec x \sim Q} \lp[ g(\vec x)  
\lp( \sgn( \vec u \cdot \vec x - t) + \sgn( \vec u \cdot \vec x + t) \rp)
\rp]\d t} \leq O(\sqrt{\eta})
\, . \end{align*}
\end{lemma}
}
\begin{proof}
We split the expectation into two parts based on the relative sizes 
of $\abs{ \vec u \cdot \vec x }$ and $t$. 
Specifically, we can write: 
\new{
\begin{align}
& \abs{\int_{ t \geq 1/\sqrt{\eta} } \E_{\vec x \sim Q} \lp[ g(\vec x)  \lp( 
\sgn( \vec u \cdot \vec x - t)
+
\sgn( \vec u \cdot \vec x + t)
\rp) \rp] \d t}     \nonumber\\
&\leq
\abs{
\int_{ t \geq 1/\sqrt{\eta} } \E_{\vec x \sim Q} \lp[ g(\vec x)  
\lp( 
\sgn( \vec u \cdot \vec x - t)
+
\sgn( \vec u \cdot \vec x + t)
\rp)
 \mathbbm 1 \{ \abs{\vec u \cdot \vec x} \geq t \}
\rp]\d t} \nonumber \\
&+
\abs{
\int_{ t \geq 1/\sqrt{\eta} } \E_{\vec x \sim Q} \lp[ g(\vec x)  
\lp( 
\sgn( \vec u \cdot \vec x - t)
+
\sgn( \vec u \cdot \vec x + t)
\rp)
  \mathbbm 1 \{ \abs{\vec u \cdot \vec x} \leq
t\}
\rp]\d t
}. \label{eq:expectation-split}
\end{align}
}
For the second term in \Cref{eq:expectation-split}, we rely on the following observation:
when $\abs{\vec u \cdot \vec x} \leq t$, 
the quantities $\vec u \cdot \vec x - t$ and $\vec u \cdot \vec x + t$ have opposite signs. 
Hence, we conclude the integrand is $0$ everywhere 
and therefore the second term is also $0$.
For the first term, we have
\new{
\begin{align*}
& \abs{
\int_{ t \geq 1/\sqrt{\eta} } \E_{\vec x \sim Q} \lp[ g(\vec x)  
\lp( \sgn( \vec u \cdot \vec x - t) + \sgn( \vec u \cdot \vec x + t) \rp)
 \mathbbm 1 \{ \abs{\vec u \cdot \vec x} \geq t \}
\rp]} \\
&\leq 
\int_{ t \geq 1/\sqrt{\eta} } \E_{\vec x \sim Q} \lp[ \bigg|g(\vec x)  
\lp( \sgn( \vec u \cdot \vec x - t) + \sgn( \vec u \cdot \vec x + t) \rp)
 \mathbbm 1 \{ \abs{\vec u \cdot \vec x} \geq t \} \bigg|
\rp] \\
&\leq
\int_{ t \geq 1/\sqrt{\eta} } 
\E_{\vec x \sim Q} \lp[ 2  \mathbbm 1 \{ \abs{\vec u \cdot \vec x} \geq t \}
\rp] 
\leq  4  \int_{ t \geq 1/\sqrt{\eta} } 
 \frac{1}{ t^2 } 
\leq O(\sqrt{\eta}) \, ,
\end{align*}
}
where the first inequality follows from the triangle inequality, \new{
the second inequality uses the fact that the $\sgn(\cdot)$ function is at most $1$}
and the \new{third} inequality follows from Chebyshev's inequality using 
the fact that the $\E[\x\x^\top]\preccurlyeq 2\vec I$.
Combining our analysis for the two terms in \Cref{eq:expectation-split}, 
we can then conclude the proof of \Cref{lem:chow-distance-tail}.
\end{proof}

Using the triangle inequality and applying \Cref{lem:chow-distance-tail} on the distributions $B$ and $\normal(\vec 0,\vec I)$, we have that  
\new{
\begin{align}
\frac{1}{2}
&\bigg| \int_{ t \geq 1/\sqrt{\eta} } \bigg(\E_{\vec x \sim B} \lp[ h(\vec x)  
\lp(\sgn( \vec u \cdot \vec x - t) + \sgn( \vec u \cdot \vec x + t)\rp)
\rp] \nonumber \nonumber \\
&- \E_{\vec x \sim \normal(\vec 0, \vec I)} \lp[ h(\vec x)  \lp(\sgn( \vec u \cdot \vec x - t) + \sgn( \vec u \cdot \vec x + t)\rp) \rp] 
\bigg)\d t
\bigg|
\leq O(\sqrt{\eta}) \;. \label{eq:large-t-bound}
\end{align}
}
We then turn our attention to the terms
\begin{equation}
\abs{ \int_{ 0\leq t \leq 1/\sqrt{\eta} }\left( \E_{\vec x \sim B} \lp[ h(\vec x)  \sgn( \vec u \cdot \vec x - t)\rp]
- \E_{\vec x \sim \normal(\vec 0, \vec I)} \lp[ h(\vec x)  \sgn( \vec u \cdot \vec x - t)  \rp]\right) \d t} \;.\label{eq:second-term} 
\end{equation}
\begin{equation}
\abs{ \int_{ 0\leq t \leq 1/\sqrt{\eta} }\left( \E_{\vec x \sim B} \lp[ h(\vec x)  \sgn( \vec u \cdot \vec x + t)\rp]
- \E_{\vec x \sim \normal(\vec 0, \vec I)} \lp[ h(\vec x)  \sgn( \vec u \cdot \vec x + t)  \rp]\right) \d t} \;.\label{eq:third-term} 
\end{equation}
To bound \Cref{eq:second-term,eq:third-term}, we need the following fact from \cite{Gollakota2022}.
\begin{fact}[Theorem 5.6 of \cite{Gollakota2022}]
\label{lem:sandwich-bound}
Let $h: \R^d \mapsto \{\pm 1\}$ be a function of $p$ halfspaces, i.e., $h(\vec x) = g\lp( h_1(\vec x), \cdots, h_p(\vec x) \rp)$ where $h_i$ are halfspaces and $g: \{\pm 1\}^p \mapsto \{\pm 1\} $.
For any $k \in \mathbb N$, 
let $ \Delta = \frac{  \sqrt{p} }{ 2k }  \frac{1}{  d^k }  \lp(  \frac{1}{ C'  \sqrt{k}  } \rp)^{k+1}$ for some sufficiently large absolute constant $C'>0$.
Then, for any distribution $B$ whose moments up to order $k$ match those of $\normal(\vec 0, \vec I)$ up to $\Delta$, we have
$$
\lp|  \E_{\vec x \sim \normal(\vec 0, \vec I)} \lp[ h(\vec x)  \rp] - \E_{\vec x \sim B} \lp[ h(\vec x) \rp] \rp|
\leq \frac{1}{\sqrt{k}} \sqrt{p}  \lp( C  \log \lp( \sqrt{p k} \rp) \rp)^{2p} \;.
$$
 for some constant $C > 0$.
\end{fact}
For a fixed $t$, note that $h(\vec x) \sgn(\vec u \cdot \vec x - t)$ is a function of two halfspaces. Moreover, from the assumptions of \Cref{lem:moment-robust-chow}, the distributions $B$ and $\normal(\vec 0,\vec I)$ match $k=C\log(1/\eta)/\eta^2$ moments up to error $\Delta =  \frac{1}{k  d^k }  \lp(  \frac{1}{ C  \sqrt{k}  } \rp)^{k+1}$, where $C>0$ is a sufficiently large absolute constant. Therefore, applying \Cref{lem:sandwich-bound} and the 
triangle inequality gives 
\begin{align}
&\frac{1}{2}
\abs{ \int_{ 0\leq t \leq 1/\sqrt{\eta} }\bigg( \E_{\vec x \sim B} \lp[ h(\vec x)  \sgn( \vec u \cdot \vec x - t) \rp]
- \E_{\vec x \sim \normal(\vec 0, \vec I)} \lp[ h(\vec x)  \sgn( \vec u \cdot \vec x - t) \rp]\bigg) \d t} \nonumber \\
&\leq O(1)  \int_{ 0\leq t \leq 1/\sqrt{\eta} } \eta\d t
= O(\sqrt{\eta}) \label{eq:small-t-bound} \;.
\end{align}
Similarly, we can show that \Cref{eq:third-term} is bounded by $O(\sqrt{\eta})$.
Substituting the bounds from \Cref{eq:large-t-bound,eq:small-t-bound} into \Cref{eq:split} then gives
$$
\abs{\E_{\vec x \sim B} \lp[ h(\vec x)  \vec x \cdot \vec u \rp]
- \E_{\vec x \sim\normal(\vec 0, \vec I)} \lp[ h(\vec x)  \vec x \cdot \vec u \rp]} \leq O(\sqrt{\eta}).
$$
Since $\vec u$ is chosen as an arbitrary unit vector, this implies that
$$
\snorm{2}{ \E_{\vec x \sim B} \lp[ h(\vec x)  \vec x  \rp]
- \E_{\vec x \sim\normal(\vec 0, \vec I)} \lp[ h(\vec x)  \vec x \rp] }\leq O(\sqrt{\eta}).
$$
Combining this with \Cref{lem:unbiased-chow} concludes the proof of \Cref{lem:moment-robust-chow}.
\end{proof}

With \Cref{lem:moment-robust-chow} in hand, we know it suffices 
to estimate the Chow parameters of $h^\ast$ with respect to $D_{\vec x}$. 
This would then give us a good approximation to $\vec v^\ast$ conditioned 
on $D_{\vec x}$ indeed having its low-degree moments approximately 
match those of $\normal(\vec 0, \vec I)$.
We use the following algorithm, which estimates the Chow parameters 
\emph{robustly} under adversarial label noise.
\begin{lemma}\label{lem:robust-chow}
Let $G$ be a distribution over $\R^d\times \{ \pm 1\}$ such that 
$\E_{\vec x \sim G_\x}[ \vec x \vec x^\top ] \preccurlyeq 2 \vec I$.
Let $\vec v\in \R^d$ be a unit vector such that 
$\vec v=\argmin_{\w\in \R^d}\pr_{(\x,y)\sim G}[\sign(\vec w\cdot\x)\neq y]$
and assume that
$ \pr_{(\x,y)\sim G}[\sign(\vec v\cdot\x)\neq y] \leq \eps. $
Then there exists an algorithm that takes $N=\poly(d,1/\eps)$ samples, 
runs in time $\poly(N)$, and outputs a vector $\vec w$ such that
$$
\snorm{2}{\E_{\vec  x \sim G_\x} \lp[   \sgn(\vec v \cdot \vec x) \vec x  \rp] - \vec w}
\leq O(\sqrt{\eps}) \;.
$$
\end{lemma}
\begin{proof}
\new{We first show that $\snorm{2}{\E_{\vec  x \sim G_\x} \lp[   \sgn(\vec v \cdot \vec x) \vec x  \rp] - \E_{(\vec  x,y) \sim G} \lp[   y \vec x  \rp] }
\leq O(\sqrt{\eps}) $.  For any unit vector $\vec u$, we have that 
\begin{align*}
    \E_{\vec  x \sim G_\x} \lp[   \sgn(\vec v \cdot \vec x) \vec u\cdot \vec x  \rp] - \E_{(\vec  x,y) \sim G} \lp[   y \vec u\cdot \vec x  \rp]&=\E_{(\vec  x,y) \sim G} \lp[   (\sgn(\vec v \cdot \vec x) -y) \vec u\cdot \vec x  \rp]
    \\&\leq \sqrt{\E_{(\vec  x,y) \sim G} \lp[   (\sgn(\vec v \cdot \vec x) -y)^2\rp]\E_{\vec  x \sim G_\x} \lp[ (\vec u\cdot \vec x)^2  \rp]}
    \\&\leq 4\sqrt{\eps}\;,
\end{align*}
where we used the \CS inequality and the fact that $ \pr_{(\x,y)\sim G}[\sign(\vec v\cdot\x)\neq y] \leq \eps$. Therefore, we have that $\snorm{2}{\E_{\vec  x \sim G_\x} \lp[   \sgn(\vec v \cdot \vec x) \vec x  \rp] - \E_{(\vec  x,y) \sim G} \lp[   y \vec x  \rp] }
\leq 4\sqrt{\eps} $.
Let $(\x^{(1)},y^{(1)}),\ldots,(\x^{(N_1)},y^{(N_1)})$ be samples drawn from $D$, where $N_1=O(d/\eps^2)$. Then, let $\widetilde{\vec w_i}=(1/N_1)\sum_{i=1}^{N_1}y^{(i)}\x^{(i)}\cdot \vec e_i$. 
From, Markov's inequality, we have that $\pr[|\widetilde{\vec w_i}-\E_{(\x,y)\sim D}[y\x]|\geq \eps/\sqrt d]\leq 4 d/(N_1\eps^2)\leq 1/4$. Therefore, using the standard median technique, we can find a $\vec w_i^{\mathrm{median}}$, so that $\pr[|\vec w_i^{\mathrm{median}}-\E_{(\x,y)\sim D}[y\x]|\geq \eps/\sqrt d]\leq \tau/d$, using $N_2=O(N_1\log(d/\tau))$ samples. Let $\vec w=(\vec w_1^{\mathrm{median}},\ldots,\vec w_d^{\mathrm{median}})$, then we have  that $\snorm{2}{\E_{(\vec  x,y) \sim G} \lp[   y \vec x  \rp] - \vec w}
\leq O(\eps)$ with probability at least $1-\tau$. Then, using the triangle inequality, we have that $\snorm{2}{\E_{\vec  x \sim G_\x} \lp[   \sgn(\vec v \cdot \vec x) \vec x  \rp] - \vec w}
\leq O(\sqrt{\eps})$, which concludes the proof of \Cref{lem:robust-chow}.}
\end{proof}
We are ready to present the algorithm and conclude the proof of \Cref{lem:constant-learn}.

\begin{proof}[Proof of \Cref{lem:constant-learn}]
Let $k, \Delta, N$ be defined as in \Cref{alg:proper-testable-learner}.
If $D_{\vec x}$ is $\normal(\vec 0, \vec I)$, the moments up to degree $k$ 
of the $\vec x$-marginal of the empirical distribution $\widehat D_N$ 
(obtained after drawing $N$ \iid samples from $D$) 
are close to those of $\normal(\vec 0, \vec I)$ 
up to additive error $\Delta$ with probability at least $1 - \tau/10$.

If \Cref{alg:proper-testable-learner} did not terminate on Line~\ref{alg:moment-test}, 
we then have that the moments up to degree $k$ of the $\vec x$-marginal of $\widehat D_N$ 
are close to those of $\normal(\vec 0, \vec I)$ up to additive error $\Delta$ 
with probability at least $1 - \tau/10$.
Let $(\widehat D_{N})_\x$ be the $\vec x$-marginal of $\widehat D_{N}$.
Then, applying \Cref{lem:moment-robust-chow} with 
$B = (\widehat D_{N})_\x$ and $h(\x) = \sign(\vec v^\ast\cdot\x)$, we get that 
\begin{align} \label{eq:moment-error}
\snorm{2}{\E_{ \vec x  \sim (\widehat D_{N})_\x } \lp[  \sign(\vec v^\ast\cdot \x)\x\rp] - \sqrt{2/\pi}\vec v^\ast} \leq O(\sqrt{\eta}).
\end{align}
By our assumption, the error of $ \sign(\vec v^\ast\cdot\x)$ under $D$ is at most $\OPT$.
Hence, the error of $\sign(\vec v^\ast\cdot\x)$ under $\widehat D_N$ is at most $\OPT + \eta$ with probability at least $1-\tau/10$.
\new{Assuming that this holds,} by \Cref{lem:robust-chow}, with probability at least $1 - \tau/10$, the vector 
$\vec w$ computed on Line~\ref{alg1:chow} of \Cref{alg:proper-testable-learner} satisfies 
\begin{align} \label{eq:agnostic-error}
 \snorm{2}{\E_{ \vec x  \sim (\widehat D_{N})_\x } \lp[  \sign(\vec v^\ast\cdot \x)\x\rp] - \vec w} \leq O\lp(\sqrt{\OPT + \eta}\rp) \;.
\end{align}
Combining \Cref{eq:moment-error,eq:agnostic-error}, we get that
$$
\snorm{2}{\vec w - \sqrt{2/\pi}\vec v^\ast} \leq O\lp( \sqrt{\OPT + \eta} \rp) \;,
$$
as desired. 
\end{proof}

\begin{Ualgorithm}
	\centering
	\fbox{\parbox{6.3in}{
			{\bf Input:} Sample access to a distribution $D$ over labeled examples; certification range $\eta$; failure probability $\tau$.\\
   {\bf Output:} \new{Either reports that the $D_\x$ is not $\normal(\vec 0,\vec I)$; or returns a unit vector $\vec w\in \R^d$ such that $\|\vec v^\ast-\vec w\|_2\leq \Cagn\sqrt{\opt+\eta}$.}
\begin{enumerate}
\item Set $k=C\log(1/\eta)/\eta^2$ and $\Delta =  \frac{1}{k  d^k }  \lp(  \frac{1}{ C  \sqrt{k}  } \rp)^{k+1}$, where $C>0$ is a sufficiently large absolute constant.
\item Draw $N=d^{C  k \log k}  \log(1/\tau)$ samples from $D$ and construct the empirical distribution $\widehat{D}_{N}$.
\item  Certify that the moments of $\widehat D_N$  up to degree $k$ match with those of $\normal(\vec 0, \vec I)$ up to error $ \Delta $. \label{alg:moment-test}
\item If the above does not hold; report that $D_\x$ is not $\normal(\vec 0, \vec I)$ and terminate.
\item Use algorithm from \Cref{lem:robust-chow} on $\widehat {D}_N$ and obtain $\vec w$. Return $\vec w/\|\vec w\|_2$. \label{alg1:chow}
\end{enumerate}
}}
\vspace{0.2cm}
\caption{Proper Testable Learner} \label{alg:proper-testable-learner}
\end{Ualgorithm}

\section{Efficient Testable Learning of Halfspaces}
In this section, we give our tester-learner 
for homogeneous halfspaces under the Gaussian distribution, thereby proving
\Cref{thm:testable-learning-agnostic}.
Throughout this section, we will fix an optimal halfspace $h^\ast(\vec x) = \sgn(\vec v^\ast \cdot \vec x)$, i.e., a halfspace with optimal 0-1 error.

The structure of this section is as follows: In \Cref{ssec:distances}, we \new{present a tester which 
certifies that the probability of the disagreement region 
between two halfspaces whose defining vectors are close to each other 
is small under $D_{\x}$}.
In \Cref{ssec:final-alg}, we \new{present and analyze our localization step 
and combine it with the tester from \Cref{ssec:distances} to obtain our final algorithm}.

\subsection{From Parameter Distance to 0-1 Error} \label{ssec:distances}

For two homogeneous halfspaces $h_{\vec u}(\vec x) = \sgn(\vec u \cdot \vec x)$ and 
$h_{\vec v}(\vec x) = \sgn(\vec v \cdot \vec x )$, where $\vec{u}, \vec{v}$ are unit vectors, 
if $D_{\vec x}$ is the standard Gaussian, $\normal(\vec 0, \vec I)$, 
we can express the probability mass of their disagreement region as follows 
(see, e.g., Lemma 4.2 of \cite{DKS18a}): 
\begin{equation} \label{eq:angle-mass-relationship}
\Pr_{ \vec x \sim D_{\vec x} } \lp[ h_{\vec u}(\vec x) \neq h_{\vec v}(\vec x) \rp] \leq O \lp(  \snorm{2}{\vec u - \vec v} \rp) \;.
\end{equation}
Hence, learning homogeneous halfspaces under Gaussian marginals 
can often be reduced to approximately learning the defining vector 
of some optimal halfspace $h^\ast$.
This is no longer the case if $D_{\vec x}$ is an arbitrary distribution, 
\new{which may well happen in our regime}. 
We show in this section that it is still possible to ``certify'' 
whether some relationship similar to the one in \Cref{eq:angle-mass-relationship} holds.

\begin{Ualgorithm}
	\centering
	\fbox{\parbox{6in}{
			{\bf Input:} Sample access to a distribution $D_{\vec x}$ over $\R^d$; tolerance parameter $\eta>0$; unit vector $\vec v \in \R^d$;
failure probability $\tau \in (0,1)$.\\
{\bf Output:} \new{ Certifies that for all unit vectors $\vec w$ such that $ \snorm{2}{ \vec w - \vec v } \leq \eta$ 
it holds that $\Pr_{ \vec x \sim D_{\vec x} }[ \sgn(\vec v \cdot \vec x) \neq \sgn(\vec w \cdot  \vec x)] \leq C  \eta$, for some absolute constant $C>1$, or reports that $D_\x$ is not $\normal(\vec 0,\vec I)$.}
\begin{enumerate}
\item Set $B = \ceil{\sqrt{\log(1/\eta)} / \eta}$.
\item Let $\widetilde D$ be the empirical distribution obtained by drawing $\poly(d, 1/\eta)  \log(1/\tau)$ many samples from $D_{\vec x}$.
\item For integers $-B-1 \leq i \leq B$, define $E_i$ to be the event  that $\{\vec v \cdot \vec x \in [ i\eta , (i+1)  \eta ]\}$ and $E_{B+1}$ to be the event  that $\{|\vec v \cdot \vec x| \geq \sqrt{\log(1/\eta)}\}$.
\item \label{line:probability-check} Verify that
$$
\sum_{i=-B-1}^{B+1}
\abs{ \Pr_{ \normal(\vec 0, \vec I) } \lp[ E_i\rp] - \Pr_{ \widetilde D } \lp[ E_i\rp]}\
\leq \eta.
$$
\item Let $S_i$ be the distribution of $\widetilde D$ conditioned on $E_i$ and $S_i^{\perp}$ be $S_i$ projected into the subspace orthogonal to $\vec v$.
\item \label{line:moment-check} For each $i$, verify that $S_i^{\perp}$ 
has bounded covariance, i.e., check that
\(
\E_{\x \sim S_i^\perp}[\x \x^\top] \preccurlyeq 2 \vec I
\,.
\)

\end{enumerate}}}
\vspace{0.2cm}
\caption{Wedge-Bound} \label{alg:wedge-bound}
\end{Ualgorithm}

In particular, given a known vector $\vec v$, we want to make sure that for any other vector $\vec w$ that is close to $\vec v$, the mass of the disagreement region between the halfspaces defined by by $\vec v, \vec w$ respectively is small.
To do so, we will decompose the space into many thin ``slabs''  that are stacked on top of each other in the direction of $\vec v$.
Then, we will certify the mass of disagreement restricted to each of the slab is not too large. 
For slabs that are close to the halfspace $\sign(\vec v \cdot \vec x)$, we can check these slabs must not themselves be too heavy. For slabs that are far away from the halfspace, we use the observation that the points in the disagreement region must then have large components in the subspace perpendicular to $\vec v$. Hence, as long as $D$ has its second moment bounded, we can bound the mass of the disagreement region in these far-away slabs using standard concentration inequality.
\begin{lemma}[Wedge Bound]
\label{lem:wedge-bound}
Let $D_{\vec x}$ be a distribution over $\R^d$. 
Given a unit vector $\vec v$ and parameters $\eta, \tau \in (0,1/2)$, 
there exists an algorithm (\Cref{alg:wedge-bound}) that draws 
\iid samples from $D_{\vec x}$, runs in time $\poly(d, 1/\eta)  \log(1 / \tau)$, 
and \new{reports either one of the following}: 
\begin{itemize}
\item[(i)] For all unit vectors $\vec w$ such that $ \snorm{2}{ \vec w - \vec v } \leq \eta$ 
it holds  $\Pr_{ \vec x \sim D_{\vec x} }[ \sgn(\vec v \cdot \vec x) \neq \sgn(\vec w \cdot  \vec x)] \leq C  \eta$, for some absolute constant $C>1$. 
\item[(ii)] $D_{\vec x}$ is not the standard Gaussian $\normal(\vec 0, \vec I)$.
\end{itemize}
\new{Moreover, with probability at least $1 - \tau$, the report is accurate.}
\end{lemma}

\begin{proof}
Recall that $\widetilde D$ is the empirical distribution made up of $N$ \iid samples from $D_{\vec x}$ where $N = \poly(d, 1/\eta) \log(1/\tau)$.
We consider $\widetilde D$ 
restricted to a set of thin ``slabs'' 
stacked on each other in the direction of $\vec v$.
More formally, we define $S_i$ to be the distribution of $\widetilde D$ 
conditioned on $ \vec v \cdot \vec x \in [ (i-1)  \eta, i  \eta ] $, 
for $i \in [- \sqrt{\log(1/\eta)}/\eta, \sqrt{\log(1/\eta)}/\eta ]$, 
and $S_i^\perp$ to be the distribution $S_i$ projected into the subspace orthogonal to $\vec v$.

Suppose that $D_{\vec x}$ is indeed $\normal(\vec 0, \vec I)$. 
Then the distribution $\widetilde D$ is the empirical distribution 
formed by samples taken from $\normal(\vec 0, \vec I)$. In this case, 
it is easy to see that both Line~\ref{line:probability-check} and ~\ref{line:moment-check} of \Cref{alg:wedge-bound}  pass with high probability.
\begin{claim}
Assume that $D_\x = \normal(\vec 0, \vec I)$. Then 
the tests at Line~\ref{line:probability-check} and \ref{line:moment-check} of \Cref{alg:wedge-bound} pass
with probability at least $1-\tau/10$.
\end{claim}
\begin{proof}
If $\vec x \sim \normal(\vec 0, \vec I)$,
then $\vec v \cdot \vec x \sim \normal(0, 1)$. 
If we concatenate the values of
$\Pr_{\normal(\vec 0, \vec I)}[E_i]$ into a vector, it can be viewed as the discretization of $\normal(0, 1)$ into $2B + 3$ many buckets.
On the other hand, $\Pr_{\tilde D}[E_i]$ is an empirical version of this discrete distribution composed of $N$ \iid samples where $N = \poly(d, 1/\eta) \log(1/\tau)$.
Since we can learn any discrete distribution with support $n$ up to error $\eta$ in total variation distance with $\Theta(n/\eta^2) \log(1/\tau)$ samples with probability at least $1 - \tau$, it follows that Line~\ref{line:probability-check} will pass with high probability  as long as we take more than $\Theta(B / \eta^2) \log(1/\tau) \leq \poly(d, 1/\eta) \log(1/\tau)$ many samples.

For Line~\ref{line:moment-check}, we remark that $S_i^{\perp}$ is the empirical version of a $(d-1)$-dimensional standard Gaussian. Since the empirical mean and the empirical covariance concentrates around the true mean and covariance with probability at least $1 - \tau$ if one takes more than $\Theta(d^2 / \eta^2) \log(1/\tau)$ many samples, it follows that Line~\ref{line:probability-check} will pass with high probability as long as we take more than $\Theta(d^2 / \eta^2) \log(1/\tau) \leq \poly(d, 1/\eta) \log(1/\tau)$ many samples.
\end{proof}
Suppose that both lines pass. We claim that this implies the following: 
for all unit vectors $\vec w$ such that $\snorm{2}{\vec w - \vec v} \leq \eta$ 
it holds
\begin{align} \label{eq:empirical-property}
    \Pr_{\vec x \sim \widetilde D} \lp[ \sgn(\vec v \cdot \vec x) \neq \sgn(\vec w \cdot \vec x) \rp]
    \leq C  \eta \;.
\end{align}
for some absolute constant $C > 1$.
Given this, we can deduce that the same equation must also hold 
for $D_{\vec x}$ with high probability ---  albeit with a larger constant $C'$.
\new{To see this, we remark that the left hand side of the equation can be treated as the error of the halfspace $\sgn(\vec w \cdot \vec x)$ if the true labels are generated by $\sgn(\vec v \cdot \vec x)$.
Since the VC-dimension of the class of homogeneous halfspaces is $d$, we have that the error for all $\vec w$ under $D_{\vec x}$ is well-approximated by that under $\widetilde D$ up to an additive $\eta$ with probability at least $1 - \tau$ given $N = \poly(d, 1/\eta) \log(1/\tau)$ many samples.
Hence, conditioned on Equation~\eqref{eq:empirical-property}, it holds with probability at least $1 - \tau$ that
\begin{align*} 
    \Pr_{\vec x \sim  D_{\vec x}} \lp[ \sgn(\vec v \cdot \vec x) \neq \sgn(\vec w \cdot \vec x) \rp]
    \leq (C+1)  \eta \;,
\end{align*}
for all $\vec w$.
}

\new{We now proceed to show Equation~\eqref{eq:empirical-property} holds if the \Cref{alg:wedge-bound} did not terminate on Lines~\ref{line:probability-check} and~\ref{line:moment-check}.}
Conditioned on Line~\ref{line:moment-check}, for any unit vector $\vec u \in \R^{d}$ that is orthogonal to $\vec v$, 
we have $\abs{\E \lp[ \vec u \cdot \vec x \rp]} \leq O(1)$ and $\Var \lp[ \vec u \cdot \vec x\rp] \leq 2$. 
Using Chebyshev's inequality, 
for any $\alpha > 0$, it holds
\begin{align} \label{eq:moment-concentration}
\Pr_{ \vec x \sim S_i^\perp  } \lp[ \lp | \vec u \cdot \vec x \rp| \geq \alpha \rp] \leq O\lp(\frac{1+\eta^2}{\alpha^2} \rp) \;.
\end{align}
We can now bound 
$\Pr_{ \vec x \sim \widetilde D } \lp[ \sgn(\vec v \cdot \vec x) \neq \sgn(\vec w \cdot  \vec x) \rp]$ 
for an arbitrary unit vector $\vec w$ satisfying $\snorm{2}{ \vec w - \vec v } \leq \eta$.
We proceed to rewrite $\vec w$ as $(1-\gamma^2)^{1/2}\vec v+ \gamma\vec u $ for some unit vector 
$\vec u\in \R^d$ that is orthogonal to $\vec v$ and $\gamma \in(0, \eta) $.
Denote $\gamma'=\gamma/(1-\gamma^2)^{1/2}$, then the event that $\sgn( \vec v \cdot \vec x ) \neq \sgn( \vec w \cdot \vec x )$ 
implies that  $ \gamma'  \abs{ \vec u \cdot \vec x } \geq \abs{ \vec v \cdot \vec x } $.
Therefore, we have that
\begin{align} \label{eq:slab-disagreement}
\Pr_{ \x \sim S_i }\lp[  \sgn(\vec v \cdot \vec x) \neq \sgn(\vec w \cdot \vec x)  \rp]
\leq 
\Pr_{ \x \sim S_i }\lp[  \gamma'  \abs{ \vec u \cdot \vec x } \geq \abs{ \vec v \cdot \vec x }  \rp]
\leq \Pr_{ \x \sim S_i^\perp }\lp[  \gamma'  \abs{ \vec u \cdot \vec x } \geq  i  \eta \rp]
\leq O \lp( \frac{1+\eta^2}{i^2} \rp) \;,
\end{align}
where in the second inequality we use the definition of $S_i$, 
and in the third inequality we use that $\gamma \leq \eta$ and \Cref{eq:moment-concentration}.
We now bound from above the total disagreement probability between $\vec w$ and $\vec v$ under $\widetilde D$. 
We have that
\begin{align*}
\Pr_{ \vec x \sim \widetilde D } \lp[ \sgn(\vec v \cdot \vec x) \neq \sgn(\vec w \cdot  \vec x) \rp]
&\leq 
\Pr_{ \vec x \sim \widetilde D } \lp[ \lp| \vec v \cdot \vec x \rp| \geq \sqrt{ \log(1/\eta) } \rp]+\Pr_{ \vec x \sim \widetilde D } \lp[ \lp| \vec v \cdot \vec x \rp| \leq \eta \rp]
\\
&+ 
\sum_{|i|>1 }^{\sqrt{ \log(1/\eta)}/\eta}
\Pr_{ \x \sim S_i }\lp[  \sgn(\vec v \cdot \vec x) \neq \sgn(\vec w \cdot \vec x)  \rp]
 \Pr_{\x \sim \widetilde D} \lp[  (i-1)  \eta \leq \vec x \cdot \vec v \leq i  \eta  \rp]
\\
&\leq 
6 \eta 
+ O \lp( \eta \rp)    \sum_{ |i|> 1 }^{  \sqrt{ \log(1/\eta) } / \eta } \frac{1+\eta^2}{i^2}
\leq O ( \eta) \, ,
\end{align*}
where we used that
$\Pr_{\x\sim \widetilde D} [
\abs{\vec x \cdot \vec v} > \sqrt{ \log(1/\eta) }
] \leq 3 \eta$ and $
\Pr_{\x\sim\widetilde D} [
\abs{\vec x \cdot \vec v} <\eta
] \leq 3 \eta$, 
since in Line~\ref{line:probability-check} we verified that the probabilities 
$\pr_{\x\sim \widetilde D}[\vec v \cdot \vec x \in [ (i-1)  \eta, i  \eta ]]$ are close to the probabilities under $\normal(\vec 0, \vec I)$ and hence bounded by $O(\eta)$, \Cref{eq:slab-disagreement}, 
and the fact that the series $\sum_{i} \frac{1}{i^2}$ is convergent and less than $\pi^2/6$.
\end{proof}

\subsection{Algorithm and Analysis: Proof of \Cref{thm:testable-learning-agnostic}} \label{ssec:final-alg}

We employ the idea of ``soft'' localization used in~\cite{DKS18a}.
In particular, given a vector $\vec v$ and a parameter $\sigma$, 
we use rejection sampling to define a new distribution $D_{\vec v, \sigma}$ 
that ``focuses'' on the region near the halfspace 
$\sgn(\vec v \cdot \vec x)$.

\begin{fact}[Rejection Sampling, Lemma 4.7 of \cite{DKS18a}]
\label{lem:rejection-sampling}
Let $D$ be a distribution on labeled examples $(\vec x, y) \in \R^d \times \{ \pm 1\}$. 
Let $\vec v \in \R^d$ be a unit vector and  $\sigma \in (0, 1)$. 
We define the distribution $D_{\vec v, \sigma}$ as follows: draw a sample $(\x, y)$ from $D$ and accept it with probability 
$ e^{
-(\vec v \cdot \vec x)^2 \cdot 
(\sigma^{-2}-1)/2 }$.
Then, $D_{\vec v, \sigma}$ is the distribution
of $(\vec x, y)$ conditional on acceptance. 
If the $\vec x$-marginal of $D$ is $\normal(\vec 0, \vec I)$, then the $\vec x$-marginals of $D_{\vec v, \sigma}$ is $\normal (\vec 0, \vec \Sigma)$, where  $\vec \Sigma = \vec I-(1- \sigma^2 )\vec v \vec v^T$.
Moreover, the acceptance probability of a point is $\sigma$.
\end{fact}

The main idea of localization is the following.
Let $\vec v$ be a vector such that $\snorm{2}{\vec v - \vec v^\ast} \leq \delta$.
Suppose that we use localization to the distribution $D_{\vec v, \delta}$.
If we can learn a halfspace with defining vector $\vec w$ 
that achieves sufficiently small constant error with respect 
to the new distribution $D_{\vec v, \delta}$, we can then combine 
our knowledge of $\vec w$ and $\vec v$ to produce a new halfspace 
with significantly improved error guarantees under the original distribution $D$.
The following lemma formalizes this geometric intuition.

\begin{lemma} \label{lem:localization}
Let $\vec v^\ast, \vec v$ be two unit vectors in $\R^d$ such that $\snorm{2}{ \vec v - \vec v^\ast } \leq \delta \leq 1/100$.
Let $\vec \Sigma = \vec I - (1 - \delta^2) \vec v \vec v^T$ and
$ \vec w $ be a unit vector such that
$ \snorm{2}{   \vec w - \frac{ \vec \Sigma^{1/2} \vec  v^\ast}{ \snorm{2}{ \vec \Sigma^{1/2} \vec  v^\ast } } } \leq \zeta\leq 1/100$.
Then it holds 
$
\snorm{2}{  \frac{ \vec \Sigma^{-1/2} \vec w }{ \snorm{2}{\vec \Sigma^{-1/2} \vec w} } - \vec v^\ast  }
\leq 5 (\delta^2+\delta \zeta).
$
\end{lemma}

\noindent \new{Before we give the proof of \Cref{lem:localization}, we 
provide a few useful remarks regarding the relevant 
parameters.}

\begin{remark}
{\em Observe that in \Cref{lem:localization} we require that the distance 
of $\vec v$ and $\vec v^\ast$ is smaller than $1/100$.  While this constant is
not the best possible, we remark that some non-trivial error is indeed necessary
so that the localization step works. 
For example, assume that $\vec v$ and $\vec v^\ast$ are orthogonal, 
i.e., $\vec v \cdot \vec v^\ast = 0$, and that
$\vec \Sigma = \vec I - (1-\xi^2) \vec v \vec v^T$ for some $\xi \in [0,1]$.
Observe that $\vec \Sigma^{1/2}$ scales vectors by a factor of $\xi$ in
the direction of $\vec v$ and leaves orthogonal directions unchanged.
Similarly, its inverse $\vec \Sigma^{-1/2}$ scales vectors by a factor of $1/\xi$
in the direction of $\vec v$ and leaves orthogonal directions unchanged.
Without loss of generality, assume that $\vec v = \vec e_1, \vec v^\ast = \vec e_2$. 
Then $\vec \Sigma^{1/2} \vec v^\ast = \vec v^\ast$.  
Moreover, assume that $\vec w = a \vec e_1 + b \vec e_2$ (with $a^2 + b^2 =1$).  
We observe that 
$\|\vec w - \vec \Sigma^{1/2} \vec v^\ast / \|\vec \Sigma^{1/2} \vec v^\ast\|_2 \|_2^2  = 
\|\vec w -\vec v^\ast\|_2^2 = 2 - 2 b$.
However, observe that $\vec s = \vec \Sigma^{-1/2} \vec w = 
(a/\xi) \vec e_1 + b \vec e_2$.   Therefore,
$ \| \vec s /\| \vec s\|_2 - \vec v^\ast\|_2^2 = 2 - 2 b/\sqrt{(a/\xi)^2 + b^2}$.
We observe that for all $\xi \in [0,1]$ it holds that $\vec s / \|\vec s\|_2$ is further away from $\vec v^\ast$ than $\vec w$, i.e., rescaling by $\vec \Sigma^{-1/2}$ worsens the error.}
\end{remark}
\begin{proof}[Proof of \Cref{lem:localization}]

Since $\snorm{2}{ \vec v  - \vec v^\ast } \leq \delta$, we can write
\begin{align} \label{eq:v*-decompose}
   \vec v^\ast = \frac{1}{\sqrt{1 + \kappa^2}} \lp( \vec v + \kappa \vec u\rp) 
\end{align}
for some $\kappa \in [0, \delta]$ and some unit vector $\vec u$ perpendicular to $\vec v$.
By definition, $\vec \Sigma^{1/2}$ shrinks in the direction of $\vec v$ by a factor of $\delta$ and leaves other orthogonal directions unchanged.
Hence, it holds
$$
\vec \Sigma^{1/2} \vec v^\ast
= 
\frac{1}{\sqrt{1 + \kappa^2}}  \lp( 
\delta \vec v + \kappa \vec u
\rp).
$$
Then, using the triangle inequality, we obtain
$$
\gamma := \snorm{2}{ \vec \Sigma^{1/2} \vec v^\ast }
\leq 
\frac{1}{\sqrt{1 + \kappa^2}} \left(\delta\snorm{2}{  \vec v} + \kappa\snorm{2}{\vec u}\right)\leq \frac{2\delta}{\sqrt{1 + \kappa^2}} 
\leq 
2 \delta\,,
$$
since $\kappa$ is upper bounded by $\delta$.
Since $\snorm{2}{   \vec w - \frac{ \vec \Sigma^{1/2} \vec  v^\ast}{ \snorm{2}{ \vec \Sigma^{1/2} \vec  v^\ast } } } \leq \zeta$, we can write
$$
\vec w =  \frac{ \vec \Sigma^{1/2} \vec  v^\ast}{ \snorm{2}{ \vec \Sigma^{1/2} \vec  v^\ast } }
+ a  \vec v + b  \vec u' \, ,
$$
for some $\abs{a}, b \in [0, \zeta]$ and $\vec u'$ perpendicular to $\vec v$. We can multiply both sides by $\gamma$ and get
$$ 
\gamma \vec w =
\vec \Sigma^{1/2} \vec v^\ast + a \gamma \vec v + b \gamma \vec u' \, , 
$$
which implies that 
\begin{align*}
\gamma  \vec \Sigma^{-1/2} \vec w 
=
\vec v^\ast + a \gamma/\delta \vec v + b \gamma \vec u' \;.
\end{align*}
Using \Cref{eq:v*-decompose}, we then have
\begin{align} \label{eq:reverse-sigma}
\gamma  \vec \Sigma^{-1/2} \vec w 
=
\lp( \frac{1}{ \sqrt{1 + \kappa^2} } + a \frac\gamma \delta \rp)  \vec v
+ \frac{ \kappa }{ \sqrt{1 + \kappa^2} } \vec u
+ b \gamma \vec u'.
\end{align}
Let $\lambda := \frac{1}{ \sqrt{1 + \kappa^2} } + a \gamma/\delta$ be the coefficient before $\vec v$. We next identify the range of $\lambda$.
\begin{claim}
    It holds that $\lambda \in [1 - \delta - 2\zeta, 1 + 2\zeta]$.
\end{claim}
\begin{proof}
Recall that $\kappa \in [0, \delta]$, $\abs{a} \in [0, \zeta]$ and $\gamma \in [0, 2 \delta]$.
If we view $\lambda$ as a function of $\kappa, a, \gamma$, 
it is minimized when 
$\kappa = \delta$, $a = - \zeta$, $\gamma = 2 \delta$, which then gives
$$
\lambda \geq
\frac{1}{ \sqrt{1 + \delta^2} }
- 2 \zeta 
\geq 1 - \delta - 2 \zeta \, ,
$$
where in the second inequality we use the fact that $\frac{1}{\sqrt{1 + x^2}} \geq 1 - x$ for $x\geq 0$. 
On the other hand, $\lambda$ is maximized when
$\kappa = 0, a = \zeta, \gamma = 2 \delta$, which gives 
$
\lambda \leq 1 + 2 \zeta.
$ 
Hence, we can conclude that $\lambda \in [1 - \delta - 2\zeta, 1 + 2\zeta]$.
\end{proof}

We multiply both sides of \Cref{eq:reverse-sigma} by $\frac{1}{ 
\lambda \sqrt{1 + \kappa^2} }$, which gives 
\begin{align*}
\frac{\gamma}{\lambda \sqrt{1 + \kappa^2}} \vec \Sigma^{-1/2} \vec w
&=
\frac{1}{ \sqrt{1 + \kappa^2} }\vec v + \frac{ \kappa }{  \lambda  \lp(1 + \kappa^2 \rp) } \vec u  + b \frac{ \kappa }{ \lambda  \sqrt{1 + \kappa^2} }   \vec u' \\
&=
\vec v^\ast +
\lp(  \frac{ \kappa }{  \lambda  \lp(1 + \kappa^2 \rp) } - \frac{\kappa}{\sqrt{1 + \kappa^2}} \rp) \vec u  + b \frac{ \gamma }{ \lambda  \sqrt{1 + \kappa^2} } \vec u' \;,
\end{align*}
where in the second equality we  use \Cref{eq:v*-decompose}.
We then bound from above and below the norm, and we get that
\[
\left| \left\|\frac{\gamma}{\lambda \sqrt{1 + \kappa^2}} \vec \Sigma^{-1/2} \vec w\right\|_2-1\right|\leq \frac{\kappa}{{1+\kappa^2}}\left|\frac{1}{\lambda}-\sqrt{1+\kappa^2}\right|+\frac{b\gamma}{\lambda\sqrt{1+\kappa^2}}\;,
\]
where we used triangle inequality.
Note that 
$$|(1/\lambda)-\sqrt{1+\kappa^2}|\leq |1/\lambda-1|+|1-\sqrt{1+\kappa^2}|\leq (\delta+2\zeta)/(1-\delta-2\zeta) +\kappa$$ 
and that $\kappa\leq \delta$. Therefore, we obtain that 
\[
\left| \left\|\frac{\gamma}{\lambda \sqrt{1 + \kappa^2}} \vec \Sigma^{-1/2} \vec w\right\|_2-1\right|\leq 4(\delta^2+\delta\zeta)\;.
\]
Let $A=\left\|\frac{\gamma}{\lambda \sqrt{1 + \kappa^2}} \vec \Sigma^{-1/2} \vec w\right\|_2$. We have that
\begin{align*}
\snorm{2}{  \frac{\vec \Sigma^{-1/2} \vec w}{ \snorm{2}{\vec \Sigma^{-1/2} \vec w} } - \vec v^\ast}&\leq \|\vec v^\ast\|_2|1-1/A|+|1/A-1|\leq 5(\delta^2+\delta\zeta)\;.
\end{align*}

This concludes the proof. 
\end{proof}

We are ready to present our main algorithm and its analysis.
At a high level, we first use \Cref{alg:proper-testable-learner} from \Cref{lem:constant-learn} 
to learn a vector $\vec v$ that is close to $\vec v^\ast$ in $\ell_2$-distance 
up to some sufficiently small constant.
Then we localize to the learned halfspace and re-apply \Cref{alg:proper-testable-learner} 
to iteratively improve $\vec v$.
In particular, we will argue that, whenever the learned halfspace is still significantly suboptimal, 
the algorithm either detects  that the underlying distribution is not Gaussian 
or keeps making improvements such that $\vec v$ gets closer to $\vec v^\ast$ 
conditioned on $\vec v$ still being sub-optimal. After at most a logarithmic number of iterations, 
we know that $\vec v$ must be close to $\vec v^\ast$, 
and we can then use \Cref{alg:wedge-bound} from \Cref{lem:wedge-bound} 
to certify that the disagreement between the learned halfspace 
and $h^\ast(\vec x)$ is small.

\begin{Ualgorithm}
	\centering
	\fbox{\parbox{6in}{
			{\bf Input:}
Sample access to a distribution $D$ over labeled examples; $\eps>0$; unit vector $\vec v \in \R^d$.\\
{\bf Output:} Either reports that $D_\x$ is not $\normal(\vec 0, \vec I)$ or computes a hypothesis $h$ such that $\pr_{(\x,y)\sim D}[h(\x)\neq y]=O(\opt)$.
\begin{enumerate}
    \item 
 {Set $\tau = ({\eps}/{(C \log(1/\eps)}))$ for a sufficiently large constant $C>0$.}
 \item {Set $\eta=1/(20000\Cagn)$ where $\Cagn$ is the constant from \Cref{lem:constant-learn}.}
 \item{Run \Cref{alg:proper-testable-learner} on $D$
with accuracy $\eta$ to obtain unit vector $\vec v^{(0)}$.\label{alg:initialization}}
 \item If \Cref{alg:proper-testable-learner} reports 
 that $D_\x$ is not $\normal(\vec 0, \vec I)$, then report it and terminate.
\item For $t = 0 \ldots \log(1/\eps)$
\label{alg:start-loop}
\begin{enumerate}
    \item  Set $\delta = (1/100)  2^{-t}$.
 \item Run \Cref{alg:localized-update} with parameters $\vec v^{(t)},\delta,\eta$ and obtain $\vec v^{(t+1)}$.\label{alg:endloop}
\end{enumerate}
\item For {$i = 0 \ldots \log(1/\eps)$}\label{alg:final loop}
\begin{enumerate}
    \item 
 Run \Cref{alg:wedge-bound} with 
$\vec v^{(i)}$ and $\eta = j  \eps$ for $j \in [1/\eps]$ on the $\vec x$-marginals of $D$.
\item Set $h^{(i)}(\x) = \sgn(   \vec v^{(i)} \cdot \x)$.\label{alg-final-final-loop}
\end{enumerate}
\item Return the halfspace $h^{(i)}$ for all $i \leq \log(1/\eps)$ with the smallest empirical error using $O(\log(1/\eps)/\eps^2)$ samples from $D$.
\end{enumerate}
}}
\vspace{0.2cm}
\caption{Testable Localization} \label{alg:raw_sieve}
\end{Ualgorithm}

\begin{Ualgorithm}
	\centering
	\fbox{\parbox{6in}{
			{\bf Input:} Sample access to a distribution $D$ over labeled examples; unit vector $\vec v \in \R^d$; parameters $\tau,\eta,\delta>0$.\\
   {\bf Output:} Either reports that $D_\x$ is not $\normal(\vec 0, \vec I)$ or  computes a unit vector $\vec v'$ such that: either $\snorm{2}{ \vec v^\ast - \vec v' } \leq \delta /2$ or $\snorm{2}{ \vec v^\ast - \vec v' }\leq O(\OPT)$
   \begin{enumerate}
\item Let $D_{\vec v, \delta}$ be the distribution obtained by running Rejection Sampling with parameters $\vec v$ and $\delta$. 
\label{line:rejection-sample}
\item Check the acceptance probability is within 
$[\delta/2, 3\delta/2 ]$. Otherwise, report $\vec x$-marginals of $D$ is not $\normal(\vec 0, \vec I)$ \label{line:acceptance-check}. 
\item Let $G$ be the distribution obtained by applying the transformation $\vec \Sigma^{-1/2}$ on the $\vec x$ marginals of $D_{\vec v, \delta}$ where $\vec \Sigma = \vec I - (1 - \delta^2) \vec v \vec v^\top.$ 
\item Run \Cref{alg:proper-testable-learner} on $G$ \new{with accuracy $\eta = 1/(20000\Cagn^2)$} to obtain a unit vector $\vec w$.
\item If the algorithm reports the marginal of $G$ is not $\normal(\vec 0, \vec I)$, terminate and report it.
\item Set $\vec v' =  \vec \Sigma^{1/2} \vec w / \snorm{2}{ \vec \Sigma^{1/2} \vec w }$ and return $\vec v'$.
\end{enumerate}
}}
\vspace{0.2cm}
\caption{Testable Localized-Update} \label{alg:localized-update}
\end{Ualgorithm}

\begin{lemma}
\label{lem:induction}
Suppose $ \snorm{2}{ \vec v^\ast - \vec v^{(t)} } \leq   \delta\leq 1/100$.
There is an algorithm (\Cref{alg:localized-update}) that with probability at least $1 - \tau$,
either (i) correctly reports that the $\vec x$-marginal of $D$ is not $\normal(\vec 0, \vec I)$ or (ii) computes a unit vector $\vec v^{(t+1)}$ so that either $\snorm{2}{ \vec v^\ast - \vec v^{(t+1)} } \leq \delta /2$ or $\snorm{2}{ \vec v^\ast - \vec v^{(t)} }\leq C\OPT$, where $C>0$ is an absolute constant.
\end{lemma}

\begin{proof}
For simplicity, we denote $\vec v^{(t)}$ as $\vec v$. 
We apply the Rejection Sampling procedure from \Cref{lem:rejection-sampling} 
in the direction of $\vec v$ with $\sigma=\delta$. 
If the $\vec x$-marginal of $D$ is $\normal(\vec 0, \vec I)$, 
the acceptance probability of $D_{\vec v, \delta}$ is exactly $\delta$. 
We then estimate the acceptance probability with accuracy $\eps$; 
if it is not lying inside the interval $[\delta/2,3\delta/2]$ 
(see Line~\ref{line:acceptance-check} of \Cref{alg:localized-update}), we report that the $\x$-marginal 
of $D$ is not standard normal and terminate.

Conditioned on the event that the algorithm did not terminate, we can sample from $D_{\vec v, \delta}$, using $O(1/\delta)$ samples from $D$. 
Note that, under the distribution $D_{\vec v, \delta}$, the error of the $\vec v^\ast$ is 
\begin{align}
   \pr_{(\x,y)\sim D_{\vec v, \delta}}[\sign(\vec v^\ast\cdot\x)\neq y]&=
   \pr_{(\x,y)\sim D}[\sign(\vec v^\ast\cdot\x)\neq y \mid (\x,y) \text{ is accepted}]\nonumber
   \\&\leq \pr_{(\x,y)\sim D}[\sign(\vec v^\ast\cdot\x)\neq y]/\pr_{(\x,y)\sim D}[ (\x,y) \text{ is accepted}]\nonumber
   \\&\leq 2\opt/\delta\;,\label{eq:optimal-value-loc}
\end{align}
where we used that the probability of the acceptance is at least $\delta/2$.
Denote by $G$ the distribution of $(
\vec \Sigma^{-1/2}\x,y)$, where $(\x,y)\sim D_{\vec v, \delta}$ and $\vec \Sigma=\vec I -(1-\delta^2)\vec v \vec v^\top$.
We note that if the $\x$-marginal of $D$ were the standard normal, then $\x$-marginal of $G$ is the standard normal. Hence, we can apply the algorithm  \Cref{alg:proper-testable-learner} from \Cref{lem:constant-learn}. Under the transformed distribution, we have that the new optimal vector $(\vec  v^\ast)':= \vec \Sigma^{1/2}\vec v^\ast/\|\vec \Sigma^{1/2}\vec v^\ast\|_2$. 

From \Cref{lem:constant-learn}, after running \Cref{alg:proper-testable-learner} on the normalized distribution $G$ with error parameter $\eta\leq 1/(2000\Cagn^2)$,
conditioned on the event that it succeeds (which happens with probability at least $1-\tau$), 
it either 
(i) reports that the $\vec x$-marginal of $G$ is not standard Gaussian 
(ii) returns a unit vector $\vec w$ 
such that 
\begin{align} \label{eq:w-bound}
\snorm{2}{\vec w - (\vec v^\ast)'} \leq \Cagn  \sqrt{ \Pr_{ (\vec x, y) \sim G } \lp[ \sign((\vec v^\ast)'\cdot\vec x) \neq y \rp] + \eta }.    
\end{align}

In case (i), we can directly report that $\vec x$-marginal of $D$ is not $\normal(\vec 0, \vec I)$ since the $\vec x$-marginal of $G$ \new{ought to be} $\normal(\vec 0, \vec I)$.

In case (ii), we claim that at least one of the following hold
(a) \new{$\vec v$ before the localized update is already good enough}, i.e. $ \snorm{2}{ \vec v^\ast - \vec v } \leq 40000 {\Cagn}^2 \opt $;
(b) it holds $\snorm{2}{ (\vec v^\ast)' - \vec w } \leq 1/100$. Suppose that (a) does not hold; we will show that it then must hold $\snorm{2}{ (\vec v^\ast)' - \vec w } \leq 1/100$.
Since $\snorm{2}{ \vec v^\ast - \vec v } \leq \delta$ and $ \snorm{2}{ \vec v^\ast - \vec v } > 40000 {\Cagn}^2 \opt $, we have that
$ \opt \leq \delta/(40000 \Cagn^2) $.
Furthermore, from \Cref{eq:optimal-value-loc} we have that  $\pr_{(\x,y)\sim D_{\vec v, \delta}}[\sign(\vec v^\ast\cdot\x)\neq y]\leq 2\opt/\delta$, it then follows that
$ \Pr_{ (\vec x, y) \sim G } \lp[\sign((\vec v^\ast)'\cdot\x) \neq y \rp] \leq 2\opt/\delta\leq 1/(20000 \Cagn^2) $.
Substituting this into \Cref{eq:w-bound} then gives
$$
\snorm{2}{\vec w - (\vec v^\ast)'}
\leq 1/100\;.
$$

Using our assumption that
$\snorm{2}{\vec v - \vec v^\ast } \leq \delta<1/100$, we 
can apply \Cref{lem:localization}, which gives that
\[\snorm{2}{  \frac{ \vec \Sigma^{-1/2} \vec w }{ \snorm{2}{\vec \Sigma^{-1/2} \vec w} } - \vec v^\ast  }
\leq  \delta/2\;.\]
Hence, we set $ \vec v^{(t+1)}=\snorm{2}{  \frac{ \vec \Sigma^{-1/2} \vec w }{ \snorm{2}{\vec \Sigma^{-1/2} \vec w} }}$ and this completes the proof.
\end{proof}

\begin{proof}[Proof of \Cref{thm:testable-learning-agnostic}] Denote by $\vec v^\ast$, a unit vector with error at most $\opt$, i.e., $\pr_{(\x,y)\sim D}[\sign(\vec v^\ast\cdot\x)\neq y]\leq \opt$.
We start by analyzing \Cref{alg:raw_sieve}. In Line \ref{alg:initialization}, \Cref{alg:raw_sieve} uses \Cref{alg:proper-testable-learner} with parameter $\eta=1/(20000\Cagn^2)$ to get a hypothesis $\vec v^{(0)}$ with small distance with $\vec v^\ast$. From \Cref{lem:constant-learn}, \Cref{alg:proper-testable-learner} either reports that $\vec x$-marginal of $D$ is not $\normal(\vec 0, \vec I)$ or outputs a vector $\vec v^{(0)}$ small distance with $\vec v^\ast$.  If \Cref{alg:proper-testable-learner} reports that the $\vec x$-marginal of $D$ is not $\normal(\vec 0, \vec I)$, we can terminate the algorithm. Conditioned on the event that the algorithm did not terminate, then we have that 
\begin{align}
\snorm{2}{\vec v^{(0)} - \vec v^\ast} \leq \Cagn  \sqrt{ \opt + \eta }. 
\end{align}
We consider two cases depending on how large the value of $\opt$ is. If $\opt > 1/(20000 \Cagn^2)$, then any unit vector achieves constant error; therefore, conditioned that the algorithm did not terminate on any proceeding test, any vector we output will satisfy the guarantees of \Cref{thm:testable-learning-agnostic}. For the rest of the proof, we consider the case where $\opt \leq 1/(20000 \Cagn^2)$. In this case, $\snorm{2}{\vec v^{(0)} - \vec v^\ast} \leq 1/100$, this means that \Cref{alg:raw_sieve} on Lines \ref{alg:start-loop}-\ref{alg:endloop} will decrease the distance between the current hypothesis and $\vec v^\ast$.

Conditioned on the event that the algorithm did not terminate at Lines \ref{alg:start-loop}-\ref{alg:endloop} of \Cref{alg:raw_sieve}, we claim that there must have some $0 \leq t^\ast \leq \log(1/\eps)$ such that $\snorm{2}{ \vec v^{(t^\ast)} - \vec v^\ast } \leq O\lp(\OPT + \eps\rp)$.
Let $t'\in \N$ be the maximum value so that $2^{-t'}/100\geq 40000\OPT$, then, for all $t\leq \min(t',\log(1/\eps))$ it holds that $2^{-t}/100\geq 40000\OPT$. From \Cref{lem:induction}, we have that for all $t\leq \min(t',\log(1/\eps))$ it holds that
\[
\snorm{2}{ \vec v^\ast - \vec v^{(t)} }\leq 2^{-t-1}/100\;.
\]
From the above, note that if $t'>\log(1/\eps)$ then 
$\snorm{2}{ \vec v^{(  \log(1/\eps) )} - \vec v^\ast } \leq \eps/100$. If $t'\leq \log(1/\eps)$, we have that $\snorm{2}{ \vec v^{(t' )} - \vec v^\ast } \leq O(\OPT)$, which proves our claim. 
It remains to show that \Cref{alg:raw_sieve} will return a vector $\vec v'$ so that $\Pr_{ (\vec x, y) \sim D }[ \sgn( \vec v' \cdot \vec x  ) \neq y ]=O(\opt+\eps)$.
From \Cref{lem:wedge-bound}, conditioned that \Cref{alg:raw_sieve} did not terminate on Lines \ref{alg:final loop}-\ref{alg-final-final-loop}, we have that for all vectors $\vec v^{(0)},\ldots,\vec v^{(\log(1/\eps))}$ generated on Lines \ref{alg:start-loop}-\ref{alg:endloop} of \Cref{alg:raw_sieve}, we have that
$$
\Pr_{ (\vec x, y) \sim D }[ \sgn( \vec v^{(t)} \cdot \vec x  ) \neq y ] 
\leq O \lp( \snorm{2}{ \vec v^{(t)} - \vec v^\ast } \rp)\;.
$$
Hence, we can conclude that $\Pr_{ (\vec x, y) \sim D }[ \sgn( \vec v^{(t')} \cdot \vec x  ) \neq y ]  \leq O \lp( \OPT + \eps \rp)$. From VC inequality, we have that with at most $\widetilde{O}(\log(1/\tau)/\eps^2)$ samples, we can estimate the empirical probabilities of $\Pr_{ (\vec x, y) \sim D }[ \sgn( \vec v \cdot \vec x  ) \neq y ] $ for all the vectors generated by our algorithm up to error $\eps$. 
Since we return the one with the smallest empirical error, the returned hypothesis has error less than $O \lp( \OPT + \eps \rp)$. To conclude the proof, note that \Cref{alg:proper-testable-learner,alg:wedge-bound,alg:localized-update} use $N=\poly(d,1/\eps)$ samples and $\poly(d,N)$ runtime and \Cref{alg:raw_sieve} use each algorithm at most $O(\log(1/\eps))$ many iterations, therefore the total sample complexity of \Cref{alg:raw_sieve} is  $N=\poly(d,1/\eps)$ with $\poly(d,N)$ runtime.
\end{proof}

\bibliographystyle{alpha}
\bibliography{mydb}
\appendix

\end{document}